\documentclass{article}

\usepackage[preprint]{neurips_2026}

\usepackage[utf8]{inputenc}
\usepackage[T1]{fontenc}

\usepackage{hyperref}
\usepackage{url}
\usepackage{booktabs}
\usepackage{amsfonts}
\usepackage{nicefrac}
\usepackage{microtype}
\usepackage{xcolor}
\usepackage{algorithm}
\usepackage{algorithmic}
\usepackage{amsmath}
\usepackage{amssymb}
\usepackage{amsthm}
\usepackage{graphicx}
\usepackage{subcaption}
\usepackage{tabularx}
 \usepackage{multirow} 
\hypersetup{colorlinks=true, citecolor=blue, linkcolor=blue, urlcolor=blue}

\newtheorem{theorem}{Theorem}[section]

\newtheorem{proposition}[theorem]{Proposition}
\newtheorem{corollary}[theorem]{Corollary}

\theoremstyle{definition}

\newtheorem{assumption}[theorem]{Assumption}

\theoremstyle{remark}

\usepackage{mdframed}

\mdfdefinestyle{thmbox}{
  backgroundcolor  = gray!9,
  linecolor        = gray!45,
  linewidth        = 0.6pt,
  roundcorner      = 3pt,
  innertopmargin   = 7pt,
  innerbottommargin= 7pt,
  innerleftmargin  = 9pt,
  innerrightmargin = 9pt,
  skipabove        = \topsep,
  skipbelow        = \topsep,
}

\surroundwithmdframed[style=thmbox]{theorem}
\surroundwithmdframed[style=thmbox]{lemma}
\surroundwithmdframed[style=thmbox]{proposition}
\surroundwithmdframed[style=thmbox]{corollary}
\surroundwithmdframed[style=thmbox]{definition}

\mdfdefinestyle{remarkbox}{
  topline           = false,
  bottomline        = false,
  rightline         = false,
  leftline          = true,
  linecolor         = gray!60,
  linewidth         = 1.5pt,
  backgroundcolor   = white,
  innertopmargin    = 3pt,
  innerbottommargin = 3pt,
  innerleftmargin   = 9pt,
  innerrightmargin  = 0pt,
  skipabove         = \topsep,
  skipbelow         = \topsep,
  nobreak          = true, 
}

\surroundwithmdframed[style=remarkbox]{remark}

\title{Correcting Source Mismatch in Flow Matching with Radial-Angular Transport}

\author{%
  Fouad Oubari \\
  Université Paris-Saclay, CNRS, ENS Paris-Saclay, Centre Borelli \\
  \texttt{oubarifouad@gmail.com}
  \And
  Mathilde Mougeot \\
  Université Paris-Saclay, CNRS, ENS Paris-Saclay, Centre Borelli \\
  ENSIIE \\
  \texttt{mathilde.mougeot@ens-paris-saclay.fr}
}

\begin{document}

\maketitle

\begin{abstract}
Flow Matching is typically built from Gaussian sources and Euclidean probability paths. For heavy-tailed or anisotropic data, however, a Gaussian source induces a structural mismatch already at the level of the radial distribution. We introduce \textit{Radial--Angular Flow Matching (RAFM)}, a framework that explicitly corrects this source mismatch within the standard simulation-free Flow Matching template. RAFM uses a source whose radial law matches that of the data and whose conditional angular distribution is uniform on the sphere, thereby removing the Gaussian radial mismatch by construction. This reduces the remaining transport problem to angular alignment, which leads naturally to conditional paths on scaled spheres defined by spherical geodesic interpolation. The resulting framework yields explicit Flow Matching targets tailored to radial--angular transport without modifying the underlying deterministic training pipeline. 

We establish the exact density of the matched-radial source, prove a radial--angular KL decomposition that isolates the Gaussian radial penalty, characterize the induced target vector field, and derive a stability result linking Flow Matching error to generation error. We further analyze empirical estimation of the radial law, for which Wasserstein and CDF metrics provide natural guarantees. Empirically, RAFM substantially improves over standard Gaussian Flow Matching and remains competitive with recent non-Gaussian alternatives while preserving a lightweight deterministic training procedure. Overall, RAFM provides a principled source-and-path design for Flow Matching on heavy-tailed and extreme-event data.
\end{abstract}
\section{Introduction}
\label{sec:introduction}

An isotropic Gaussian source distribution is a default design choice in many modern generative models. It underlies diffusion models~\citep{sohl2015deep,ho2020ddpm,song2021sde}, flow-based generative models~\citep{rezende2015variational,dinh2016density,chen2018neural}, and more recently Flow Matching (FM)~\citep{lipman2023flowmatching}. This choice is attractive for analytical and computational reasons, but it can impose a structural bias when the target distribution exhibits non-Gaussian radial behavior.

Such a mismatch is not merely cosmetic. In many heavy-tailed and anisotropic settings, the distribution of norms departs substantially from that of a standard Gaussian~\citep{cont2001empirical,papalexiou2013extreme}. In high dimensions, norm statistics strongly shape the geometry of probability mass. In particular, for an isotropic Gaussian in ambient dimension $d$, most of the mass concentrates in a thin annulus at radius of order $\sqrt{d}$~\citep{vershynin2020high}. Consequently, when the source imposes Gaussian norm statistics while the target does not, the learned transport must first correct an artificial radial discrepancy before modeling the structure that actually characterizes the data. This burdens the transport with an avoidable radial correction and leads to a geometry that is less well aligned with the target distribution.

This issue is especially relevant in Flow Matching~\citep{lipman2023flowmatching}, where the generative design is specified directly through a source distribution and conditional probability paths, and the associated vector field is learned by regression. In this setting, source mismatch is not merely inherited from a forward noising process: it enters directly through the pair consisting of the source and the conditional transport geometry. This makes FM a particularly natural framework in which to study whether part of the modeling burden can be removed already at initialization, before learning the remaining transport. Recent work on multiplicative diffusion models~\citep{gruhlke2025multiplicative} similarly questions the suitability of Gaussian latent structure for heavy-tailed or anisotropic data. Our perspective is complementary. Rather than modifying the stochastic noising dynamics, we ask whether the same issue can be addressed directly within the standard simulation-free Flow Matching template through a coupled design of the source and the conditional path.

In this work, we propose \textit{Radial--Angular Flow Matching (RAFM)}, a structured Flow Matching framework that separates radial and angular roles in the transport. RAFM first matches the data radial law at the source, thereby removing the radial part of the Gaussian source mismatch by construction. Once radii are matched, the residual transport problem is primarily angular, which leads naturally to conditional paths on scaled spheres based on spherical geodesic interpolation. This yields explicit Flow Matching targets for radius-preserving conditional transport while preserving the standard FM training pipeline. Spherical geometry thus enters only through the matched-radius conditional path, while the overall generative problem remains posed in the ambient Euclidean space.

Our contributions are fourfold. First, we formalize \emph{Gaussian radial mismatch} in Flow Matching and show, via a radial--angular KL decomposition, that matching the source norm law removes the radial penalty induced by a standard Gaussian source. Second, we derive the corresponding matched-radius conditional transport and obtain explicit Flow Matching targets based on spherical geodesic interpolation. Third, we analyze the resulting dynamics, including norm preservation under tangential flows and a stability result linking Flow Matching error to generation error. Fourth, we study empirical estimation of the radial law, for which Wasserstein and CDF metrics provide natural guarantees, and evaluate the resulting framework on synthetic heavy-tailed and real structured datasets. Empirically, RAFM substantially improves over standard Gaussian FM and remains competitive with recent non-Gaussian alternatives while preserving the lightweight simulation-free FM template.

Overall, our work suggests that in Flow Matching, the source distribution should not be treated as a neutral implementation detail. For non-Gaussian targets, source design and conditional transport geometry should be chosen jointly, as part of the geometry of the generative problem itself.

\begin{figure}[t]
    \centering
    \begin{minipage}[c]{0.65\textwidth}
        \centering
        \includegraphics[width=\textwidth]{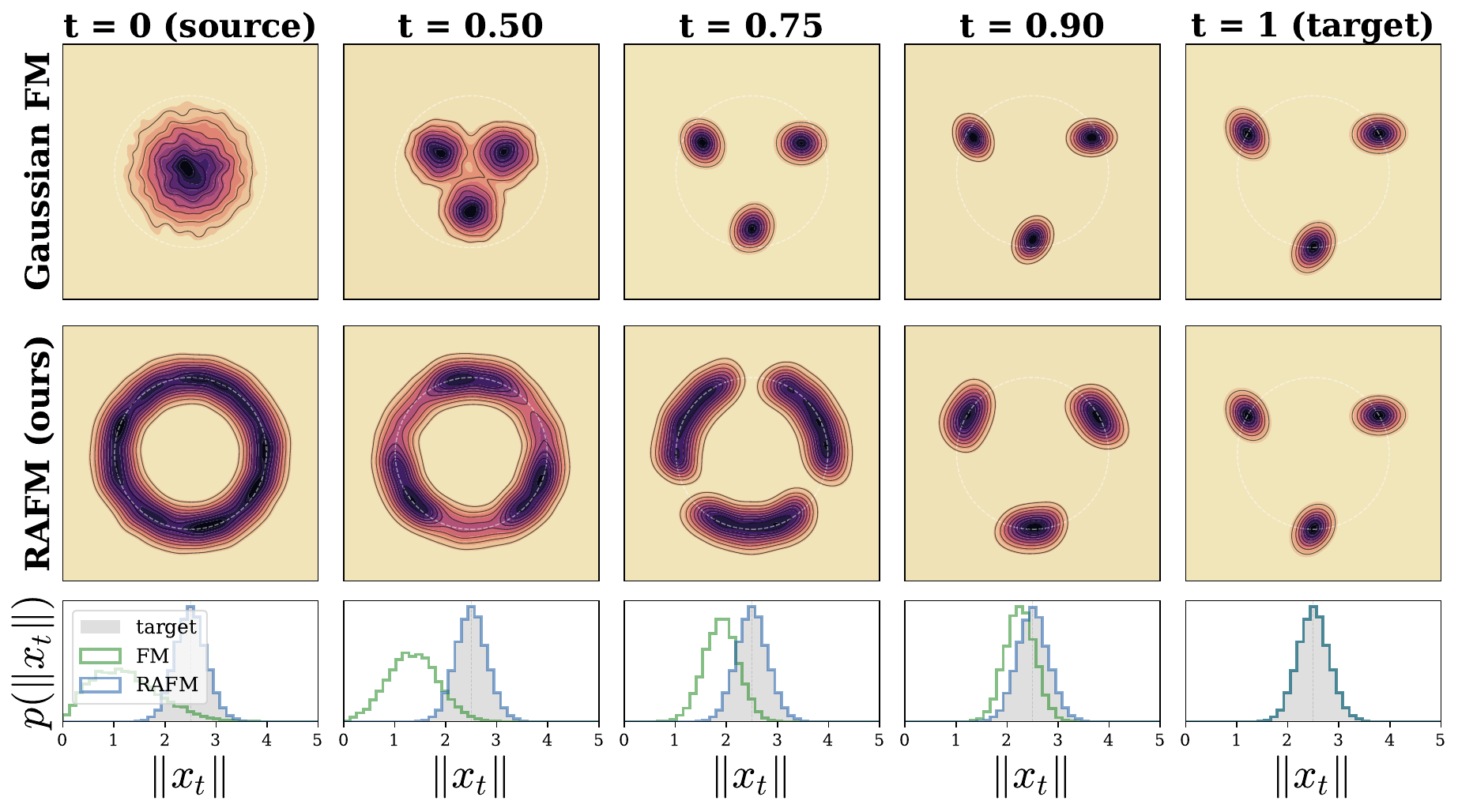}
    \end{minipage}%
    \hfill
    \begin{minipage}[c]{0.35\textwidth}
        \centering
        \includegraphics[width=\textwidth]{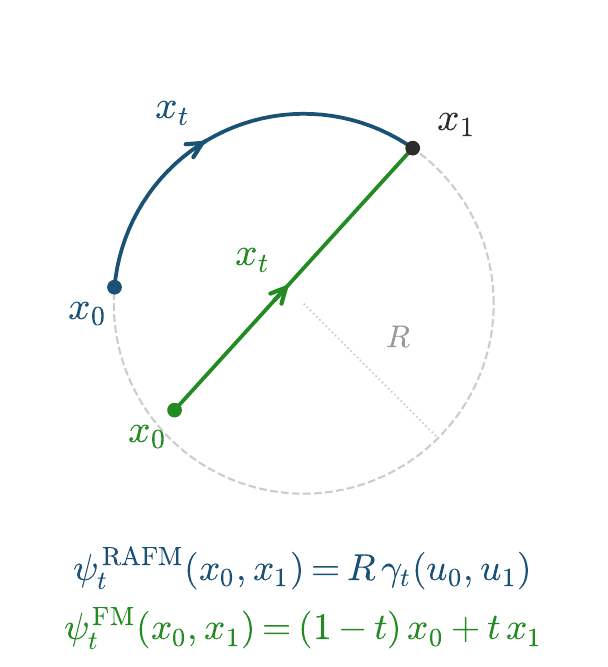}
    \end{minipage}
    \caption{Conceptual comparison between standard Gaussian Flow Matching and Radial--Angular Flow Matching (RAFM). 
    Left: intermediate marginals and radial diagnostics on the 2D toy example. Gaussian FM progressively corrects both radius and angle, whereas RAFM preserves the radial law and mainly reorganises mass angularly. 
    Right: schematic illustration of the conditional interpolation geometries used in our comparison. RAFM follows a matched-radius geodesic path on the scaled sphere, whereas the Gaussian FM baseline uses a linear Euclidean interpolation.}
    \label{fig:rafm_overview}
\end{figure}
\section{Related Work}

The works most relevant to RAFM fall into four connected directions: Flow Matching and conditional path design, manifold-aware transport, adaptation of the source distribution, and non-Gaussian diffusion dynamics. RAFM is related to each of these lines, but differs in its central objective: it addresses a source-mismatch mechanism specific to Flow Matching, and derives from it a coupled design of the source distribution and conditional transport geometry within the standard simulation-free Conditional Flow Matching template.

Flow Matching \citep{lipman2023flowmatching} introduced a simulation-free framework for training continuous normalizing flows by regressing vector fields associated with prescribed conditional probability paths. Subsequent extensions, including Rectified Flow \citep{liu2023rectifiedflow}, Simulation-Free Schr\"odinger Bridges via Score and Flow Matching \citep{tong2023simulation}, Functional Flow Matching \citep{kerrigan2023functional}, and Optimal Flow Matching \citep{kornilov2024optimal}, further showed that the choice of probability path can strongly affect learnability and sampling efficiency. RAFM builds on this path-design perspective, but focuses on a different question: when the source itself is mismatched, part of the transport burden is artificial. Our contribution is therefore not to introduce a non-Euclidean path in isolation, but to show that correcting the radial source mismatch in FM induces a matched-radius conditional transport problem whose natural geometry is angular and radius-preserving.

A second related line studies generative modeling under non-Euclidean geometry. Normalizing Flows on Tori and Spheres \citep{rezende2020normalizing}, Moser Flow \citep{rozen2021moser}, Matching Normalizing Flows and Probability Paths on Manifolds \citep{ben2022matching}, Riemannian Score-Based Generative Modelling \citep{de2022riemannian}, Flow Matching on general geometries \citep{chen2023flow}, and Metric Flow Matching \citep{kapusniak2024metric} show that non-Euclidean interpolants and manifold-aware constructions can lead to more meaningful transport than standard Euclidean paths. RAFM is related to this principle through its use of spherical geometry, but differs in scope and motivation. RAFM does not assume that the data are supported on a fixed manifold, nor does it endow the ambient data space with a global non-Euclidean geometry. Instead, the ambient distribution remains fully Euclidean, and spherical geometry appears only conditionally, after radius matching, as the transport geometry induced by the residual angular problem.

Another nearby literature modifies the source or base distribution rather than only the transport map. In normalizing flows, the classical formulation starts from a simple tractable base and learns an expressive transformation \citep{rezende2015variational}, but several works have shown that the base distribution can itself be a source of mismatch. Tails of Lipschitz Triangular Flows \citep{jaini2020tails} analyzed how tail behavior constrains what common flow architectures can represent from a given source. Resampling Base Distributions of Normalizing Flows \citep{stimper2022resampling} addressed support and topology mismatch through learned rejection-sampling bases, while Marginal Tail-Adaptive Normalizing Flows \citep{laszkiewicz2022marginal} and Flexible Tails for Normalizing Flows \citep{hickling2024flexible} proposed mechanisms to better capture heavy-tailed behavior. RAFM is closely related in spirit to this line of work, but in a more structured way: rather than enriching the source generically, it isolates the radial component of the mismatch, corrects it explicitly, and leaves the remaining discrepancy to angular transport. Our radial--angular KL decomposition further makes this reduction explicit.

Recent work has also revisited the stochastic dynamics used in diffusion models. Standard additive-Gaussian formulations such as DDPM \citep{ho2020ddpm} and the score-SDE framework \citep{song2021sde} rely on Gaussian priors and Gaussian forward noising, while related approaches such as Diffusion Schr\"odinger Bridges \citep{de2021diffusion} and later design-space analyses \citep{karras2022elucidating} explore alternative stochastic transport mechanisms and samplers. More recently, Heavy-Tailed Diffusion Models \citep{pandey2024heavy} and Multiplicative Diffusion Models \citep{gruhlke2025multiplicative} directly question the suitability of Gaussian latent structure for heavy-tailed data. The latter is the closest conceptual comparator to RAFM. Multiplicative diffusion addresses radial mismatch by modifying the forward stochastic process and learning the resulting score field, whereas RAFM addresses the same broad issue directly at the level of Flow Matching design: it keeps the standard simulation-free CFM training template and incorporates non-Gaussian structure through a coupled choice of source, coupling, and conditional path. In this sense, RAFM provides a deterministic route to modeling non-Gaussian radial structure without introducing a new stochastic noising mechanism.
\section{Method}
\label{sec:method}

RAFM specializes Conditional Flow Matching (CFM) to exploit radial--angular structure in the data. Standard Gaussian CFM combines two tasks in a single transport: it must first correct the radial mismatch induced by the Gaussian source and then model the directional structure of the target distribution. For data with non-Gaussian norm statistics, such as heavy-tailed or anisotropic distributions, this coupling can introduce an avoidable radial correction into the learned transport.

Our key idea is to separate these two roles. As illustrated in Figure~\ref{fig:rafm_overview}, RAFM first matches the data radial law at the source and then transports mass only along scaled spheres. Concretely, for a target sample $x_1$, RAFM draws
\[
x_0=\|x_1\|u_0,
\qquad
u_0\sim \mathrm{Unif}(S^{d-1}),
\]
and connects $x_0$ to $x_1$ through a spherical geodesic at fixed radius:
\[
\psi_t(x_0,x_1)=\|x_1\|\,\gamma_t(u_0,u_1),
\qquad
u_1=\frac{x_1}{\|x_1\|}.
\]
Training still uses the standard CFM regression objective, but with a source and path adapted to this radial--angular factorization. At sampling time, radii are initialized from the empirical radial law estimated on the training set, while directions are sampled uniformly.

\subsection{Conditional Flow Matching background}
\label{subsec:cfm}

We briefly recall the CFM template on which RAFM is built \cite{lipman2023flowmatching}. Let $q_0$ be a source distribution on $\mathbb{R}^d$, let $p_{\mathrm{data}}$ denote the data distribution, and let $\eta$ be a coupling of $q_0$ and $p_{\mathrm{data}}$. Given a differentiable interpolation map
\[
\psi_t:\mathbb{R}^d\times\mathbb{R}^d \to \mathbb{R}^d,
\qquad t\in[0,1],
\]
such that $\psi_0(x_0,x_1)=x_0$ and $\psi_1(x_0,x_1)=x_1$, we define
\[
(X_0,X_1)\sim \eta,
\qquad
X_t:=\psi_t(X_0,X_1).
\]
The induced marginal path $p_t=\mathrm{Law}(X_t)$ connects $q_0$ to $p_{\mathrm{data}}$.

A time-dependent vector field $v:[0,1]\times\mathbb{R}^d\to\mathbb{R}^d$ generates the flow ODE
\[
\dot Y_t = v(t,Y_t),
\qquad
Y_0\sim q_0.
\]
We distinguish this ODE state $Y_t$ from the conditional interpolation variable
$X_t=\psi_t(X_0,X_1)$. If a vector field $v$ generates the marginal path
$p_t=\mathrm{Law}(X_t)$, then the associated densities satisfy the continuity equation
\[
\partial_t p_t(x) + \nabla\!\cdot\!\bigl(p_t(x)\,v(t,x)\bigr)=0.
\]

In CFM, the target vector field associated with $\psi_t$ is
\[
v^\star(t,x)
=
\mathbb{E}\!\left[
\partial_t \psi_t(X_0,X_1)
\,\middle|\,
X_t=x
\right],
\]
and a neural vector field $v_\theta$ is trained by regressing the analytic path derivative:
\[
\mathcal{L}_{\mathrm{CFM}}(\theta)
=
\mathbb{E}_{t\sim \mathrm{Unif}[0,1]}
\mathbb{E}_{(X_0,X_1)\sim \eta}
\left[
\left\|
v_\theta(t,X_t)-\partial_t \psi_t(X_0,X_1)
\right\|_2^2
\right],
\qquad
X_t=\psi_t(X_0,X_1).
\]

Within this framework, the main design choice is the pair $(q_0,\psi_t)$. RAFM changes exactly these two objects: it matches the data radial law at the source and chooses a conditional path that preserves radius throughout the transport.

\subsection{Matching the radial law at the source}
\label{subsec:radial_source}

The first question is whether the Gaussian source creates a meaningful mismatch in the first place. When the target radial law differs from the Gaussian one, standard CFM must spend part of its transport budget correcting the norm distribution before it can model the structure that actually distinguishes the data. RAFM removes this mismatch directly at initialization.

Let
\[
S^{d-1}:=\{u\in\mathbb{R}^d:\|u\|=1\}
\]
denote the unit sphere, and let $|S^{d-1}|$ denote its surface measure. For a sample $X\sim p_{\mathrm{data}}$, assume for simplicity that $p_{\mathrm{data}}(\{0\})=0$, and write
\[
R:=\|X\|\in \mathbb{R}_+,
\qquad
U:=\frac{X}{\|X\|}\in S^{d-1}.
\]
We denote by $p_R$ the density of the radial variable $R$, and by $p_{U\mid R}(\cdot\mid r)$ the conditional angular distribution at radius $r$.

RAFM uses a source that preserves the data norm law while removing directional structure:
\[
X_0 := R\,U_0,
\qquad
R\sim p_R,
\qquad
U_0\sim \mathrm{Unif}(S^{d-1}),
\]
where $R$ and $U_0$ are independent. We denote the law of $X_0$ by $q_{\mathrm{rad}}$. By the polar change-of-variables formula,
\[
q_{\mathrm{rad}}(x)
=
\frac{p_R(\|x\|)}{|S^{d-1}|\,\|x\|^{d-1}},
\qquad x\neq 0,
\]
and, by construction,
\[
\|X_0\|\overset{d}{=}\|X\|.
\]
Hence the source preserves exactly how much mass lies at each distance from the origin, while redistributing that mass uniformly over the corresponding sphere.

This construction is closely related in spirit to the non-Gaussian latent structure induced by multiplicative diffusion \cite{gruhlke2025multiplicative}, but here it enters directly through the source choice within CFM rather than through a forward stochastic process.

To quantify the benefit of this source correction, we compare the KL divergence from $p_{\mathrm{data}}$ to $q_{\mathrm{rad}}$ and to a standard Gaussian source. Let $\phi_d$ denote the standard Gaussian density on $\mathbb{R}^d$, and let $p_{\chi_d}$ denote the density of $\|Z\|$ for $Z\sim\mathcal N(0,I_d)$. Since both $q_{\mathrm{rad}}$ and $\phi_d$ are conditionally uniform in direction at fixed radius, the difference between them is purely radial.

\begin{theorem}[Radial KL decomposition]
\label{thm:radial_kl}
Assume that the relevant conditional densities exist and that the divergences below are finite. Then
\[
\mathrm{KL}(p_{\mathrm{data}}\|q_{\mathrm{rad}})
=
\mathbb{E}_{r\sim p_R}
\Big[
\mathrm{KL}\bigl(p_{U\mid R}(\cdot\mid r)\,\|\,\mathrm{Unif}(S^{d-1})\bigr)
\Big],
\]
whereas
\[
\mathrm{KL}(p_{\mathrm{data}}\|\phi_d)
=
\mathrm{KL}(p_R\|p_{\chi_d})
+
\mathbb{E}_{r\sim p_R}
\Big[
\mathrm{KL}\bigl(p_{U\mid R}(\cdot\mid r)\,\|\,\mathrm{Unif}(S^{d-1})\bigr)
\Big].
\]
Consequently,
\[
\mathrm{KL}(p_{\mathrm{data}}\|q_{\mathrm{rad}})
\le
\mathrm{KL}(p_{\mathrm{data}}\|\phi_d),
\]
with equality if and only if $p_R=p_{\chi_d}$ almost everywhere.
\end{theorem}

Theorem~\ref{thm:radial_kl} shows that a Gaussian source pays an additional radial penalty $\mathrm{KL}(p_R\|p_{\chi_d})$, whereas the radial source removes it by construction. In other words, once the source is matched in radius, the remaining mismatch is purely angular. Proofs and extensions are deferred to Appendix~\ref{app:radial_source}.

In practice, the radial law is unknown and must be estimated from training data. Given samples
$\{x^{(i)}\}_{i=1}^n$, we form the radii
\[
r_i := \|x^{(i)}\|,
\qquad i=1,\dots,n,
\]
and define the empirical radial measure
\[
\widehat\mu_R := \frac{1}{n}\sum_{i=1}^n \delta_{r_i},
\]
with empirical CDF $\widehat F_R$. Unconditional samples are then initialized from
\[
\widehat X_0 = \widehat R\,U_0,
\qquad
\widehat R \sim \widehat\mu_R,
\qquad
U_0\sim \mathrm{Unif}(S^{d-1}),
\]
with $\widehat R$ and $U_0$ independent. In practice, $\widehat R$ may be sampled by inversion or
resampling from the empirical radial law. This is the main practical payoff of the decomposition:
source adaptation reduces to estimating a one-dimensional radial distribution. Appendix~\ref{app:empirical_radial_source}
provides uniform CDF convergence and Wasserstein transfer guarantees for this empirical source.

\subsection{Spherical transport for angular alignment}
\label{subsec:spherical_paths}

Once the radial mismatch is removed at the source, the remaining transport problem is angular. The conditional path should therefore preserve the matched radius rather than spending transport effort changing it again. RAFM achieves this by transporting samples along spherical geodesics on the scaled sphere determined by the target radius.

Given a target sample $x_1\neq 0$, let
\[
R:=\|x_1\|,
\qquad
u_1:=\frac{x_1}{\|x_1\|}\in S^{d-1}.
\]
We sample
\[
u_0\sim \mathrm{Unif}(S^{d-1}),
\qquad
x_0:=R\,u_0.
\]
By construction, $x_0$ has the same radius as $x_1$, and marginally $x_0\sim q_{\mathrm{rad}}$.

For non-antipodal directions $u_0\neq -u_1$, define
\[
\theta:=\arccos(\langle u_0,u_1\rangle)\in[0,\pi).
\]
The spherical geodesic interpolation is
\[
\gamma_t(u_0,u_1)
=
\frac{\sin((1-t)\theta)}{\sin\theta}\,u_0
+
\frac{\sin(t\theta)}{\sin\theta}\,u_1,
\qquad t\in[0,1],
\]
and the corresponding interpolation in $\mathbb{R}^d$ is
\[
\psi_t(x_0,x_1)
=
R\,\gamma_t(u_0,u_1).
\]
Degenerate antipodal and near-origin cases are measure-zero or numerically delicate edge cases; we specify deterministic completions and a dedicated failure-mode analysis in Appendix~\ref{app:spherical_path} and Appendix~\ref{app:toy2d}.

The key geometric property is that this path never changes the radius.

\begin{proposition}[Radius preservation and tangency of the spherical path]
\label{prop:spherical_path}
For any non-antipodal pair $(x_0,x_1)$ with $\|x_0\|=\|x_1\|=R$, the path
\[
X_t:=\psi_t(x_0,x_1)
\]
satisfies
\[
X_0=x_0,
\qquad
X_1=x_1,
\qquad
\|X_t\|=R
\quad \forall t\in[0,1].
\]
Moreover, its velocity is tangent to the scaled sphere $R\,S^{d-1}$:
\[
X_t^\top \dot X_t = 0
\qquad \forall t\in[0,1].
\]
\end{proposition}

Thus, once a training pair is radius-matched, the ideal transport does not need to correct the norm at all. The proof is given in Appendix~\ref{app:spherical_path}.

Differentiating the interpolation yields the analytic target velocity
\[
\partial_t \psi_t(x_0,x_1)
=
R\,\frac{\theta}{\sin\theta}
\left[
-\cos((1-t)\theta)\,u_0
+
\cos(t\theta)\,u_1
\right].
\]
This velocity is tangent by construction and corresponds to constant-speed geodesic motion on the scaled sphere. Appendix~\ref{app:spherical_path} further shows that it can be written through the Riemannian logarithm map on $R\,S^{d-1}$.

\subsection{Training and sampling with tangential constraints}
\label{subsec:practical_rafm}

Specializing CFM to the matched-radius coupling above yields the RAFM objective

\[
\begin{split}
\mathcal{L}_{\mathrm{RAFM}}(\theta)
&=
\mathbb{E}_{t\sim \mathrm{Unif}[0,1]}
\mathbb{E}_{\substack{X_1\sim p_{\mathrm{data}}\\ U_0\sim \mathrm{Unif}(S^{d-1})}}
\left[
\left\|
v_\theta(t,X_t)-\partial_t\psi_t(X_0,X_1)
\right\|_2^2
\right],\\
&\qquad
X_0=\|X_1\|U_0,
\qquad
X_t=\psi_t(X_0,X_1).
\end{split}
\]
where $X_1\sim p_{\mathrm{data}}$ and $X_0=\|X_1\|U_0$ with $U_0\sim \mathrm{Unif}(S^{d-1})$.

This formulation has two important practical consequences. First, during training, the radius is copied directly from the target sample through the matched-radius coupling, so the empirical radial law is not needed inside the loss. Second, at unconditional sampling time, the empirical radial law is used only to initialize the source radius, after which the learned dynamics transport directions on the corresponding sphere.

The target field is tangent by construction, but the learned network need not be exactly tangent. In practice, even small radial components can accumulate during numerical integration. We therefore project the predicted velocity onto the tangent space of the current sphere:
\[
\Pi_{T_x}(v)
=
v-\frac{\langle x,v\rangle}{\|x\|^2}x.
\]
This projection is not a cosmetic implementation detail. It is the practical bridge between the ideal spherical geometry of the target field and the approximate vector field learned by the network. Table~\ref{tab:full_projection_ablation} shows that it becomes increasingly important on the more challenging regimes, while Appendix~\ref{app:toy2d} isolates a distinct near-origin failure mode in very low dimension.

The resulting training and sampling procedures are summarized in Algorithms~\ref{alg:rafm_train} and~\ref{alg:rafm_sample}.

\begin{algorithm}[t]
\caption{RAFM training}
\label{alg:rafm_train}
\small
\begin{algorithmic}[1]
\REQUIRE Training set $\{x^{(i)}\}_{i=1}^n$, vector field $v_\theta$
\REPEAT
    \STATE Sample a mini-batch $\{x_1^{(b)}\}_{b=1}^B$ from the training set and times $\{t_b\}_{b=1}^B$ with $t_b\sim\mathrm{Unif}[0,1]$
    \STATE Sample $\{u_0^{(b)}\}_{b=1}^B\sim \mathrm{Unif}(S^{d-1})$ and set $x_0^{(b)}\gets \|x_1^{(b)}\|\,u_0^{(b)}$
    \STATE Compute $x_{t_b}^{(b)}\gets\psi_{t_b}(x_0^{(b)},x_1^{(b)})$ and $\dot x_{t_b}^{(b)}\gets\partial_t\psi_{t_b}(x_0^{(b)},x_1^{(b)})$
    \STATE Update $\theta$ on
    \[
    \frac{1}{B}\sum_{b=1}^B
    \bigl\|
    v_\theta(t_b,x_{t_b}^{(b)})-\dot x_{t_b}^{(b)}
    \bigr\|_2^2
    \]
\UNTIL{convergence}
\end{algorithmic}
\end{algorithm}

\begin{algorithm}[t]
\caption{RAFM sampling}
\label{alg:rafm_sample}
\small
\begin{algorithmic}[1]
\REQUIRE Trained $v_\theta$, empirical radial law $\widehat\mu_R$ (with CDF $\widehat F_R$), ODE solver
\STATE Sample $R\sim\widehat\mu_R$ and $u_0\sim\mathrm{Unif}(S^{d-1})$; set $x_0\gets R\,u_0$
\STATE Initialize the solver state with $x_t \gets x_0$
\FOR{solver steps from $t=0$ to $t=1$}
    \STATE Evaluate $\hat v\gets v_\theta(t,x_t)$
    \STATE Project $\hat v^\perp\gets \Pi_{T_{x_t}}(\hat v)$
    \STATE Advance the solver state $x_t$ using $\hat v^\perp$
\ENDFOR
\RETURN final state $x_{t=1}$
\end{algorithmic}
\end{algorithm}

\subsection{Guarantees for radial preservation and generation stability}
\label{subsec:theory_summary}

The design above raises two natural questions. First, if the learned dynamics are tangent, do they preserve the radial law fixed by the source? Second, if the learned field approximates the RAFM target well, does this translate into accurate generation? The next two results answer these questions.

\begin{proposition}[Tangential flows preserve the radial law]
\label{prop:tangential_preserves_norm_main}
Assume that
\[
x^\top v_\theta(t,x)=0
\qquad
\text{for all } (t,x)\in [0,1]\times(\mathbb{R}^d\setminus\{0\}).
\]
Let $Y_t$ solve
\[
\dot Y_t=v_\theta(t,Y_t)
\]
with $Y_0\sim q_{\mathrm{rad}}$. Then
\[
\|Y_t\|\overset{d}{=}\|Y_0\|
\overset{d}{=}\|X_{\mathrm{data}}\|,
\qquad X_{\mathrm{data}}\sim p_{\mathrm{data}},
\]
for every $t\in[0,1]$.
\end{proposition}

Proposition~\ref{prop:tangential_preserves_norm_main} formalizes the role of tangential projection: when the learned field is tangent, the norm is preserved exactly, so the radial law remains controlled entirely by the source. Proofs and stronger statements on norm evolution are given in Appendix~\ref{app:stability}.



To relate target approximation to generation quality, define the population RAFM regression error
\[
\mathcal{E}_{\mathrm{RAFM}}(\theta)
=
\mathbb{E}\int_0^1
\bigl\|
v_\theta(t,X_t)-\partial_t\psi_t(X_0,X_1)
\bigr\|_2^2\,dt,
\qquad
X_t=\psi_t(X_0,X_1).
\]

\begin{theorem}[Generation stability]
\label{thm:generation_stability_main}
Assume that $\mathbb E[\|X_1\|^2]<\infty$ and that $v_\theta$ is Lipschitz in space with constant $L_\theta$. Then
\[
W_2\!\left((\Phi_1^\theta)_\# q_{\mathrm{rad}},\,p_{\mathrm{data}}\right)
\le
e^{L_\theta}\,\mathcal{E}_{\mathrm{RAFM}}(\theta)^{1/2},
\]
where $\Phi_1^\theta$ is the flow map induced by $v_\theta$.
\end{theorem}

Theorem~\ref{thm:generation_stability_main} should be read primarily as a stability result: accurate regression of the RAFM target field implies accurate generation in Wasserstein distance, with a sensitivity controlled by the regularity of the learned flow. The factor $e^{L_\theta}$ is the standard Gr\"onwall-type amplification term that appears when propagating vector-field approximation errors through a Lipschitz ODE flow. In particular, it quantifies how local regression errors may grow along trajectories under the learned dynamics. The theorem therefore clarifies how target-field approximation error translates into generation error, with the flow regularity determining the degree of amplification along trajectories. Combined with Proposition~\ref{prop:tangential_preserves_norm_main}, the result also clarifies the radial--angular factorization of RAFM: the source fixes the radial law, the ideal path is tangent to matched-radius spheres, and the remaining approximation burden is primarily angular. Full proofs are given in Appendix~\ref{app:stability}.

\section{Experiments}
\label{sec:experiments}

\begin{figure}[t]
    \centering
    \includegraphics[width=0.8\linewidth]{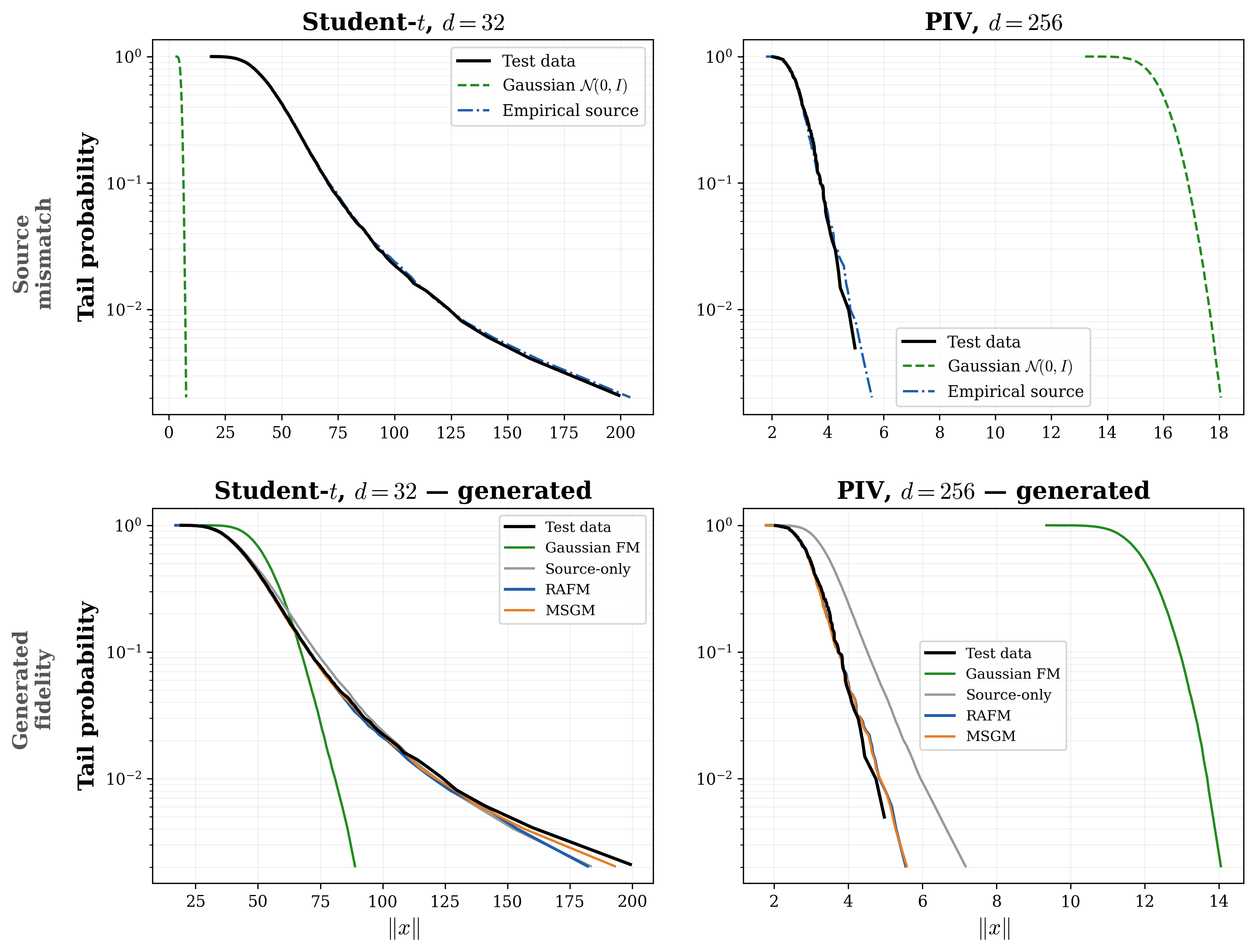}
    \caption{Radial source mismatch and generated radial fidelity on two representative hard regimes. The top row compares the test radial law with the Gaussian and empirical radial sources; the Gaussian source is strongly mismatched, whereas the empirical source closely follows the data. The bottom row shows the radial tail distributions of generated samples: Gaussian FM inherits poor radial fidelity, source-only already recovers much of the gap, and RAFM further improves the match to the target radial law. On Student-$t$ ($d=32$), MSGM remains slightly stronger on radial fidelity, whereas on PIV ($d=256$), RAFM outperforms MSGM while remaining substantially faster in wall-clock time.}
    \label{fig:radial_tail_diagnostics}
\end{figure}

We evaluate whether adapting both the source distribution and the conditional path improves Flow Matching on non-Gaussian data. Our experiments are designed to isolate two effects: the benefit of correcting the radial law at the source, and the additional benefit of radius-preserving spherical transport once radii are matched. We therefore compare standard Gaussian Flow Matching, source-corrected variants, and the recent multiplicative diffusion baseline MSGM.

\subsection{Experimental setup}

We consider both synthetic and real datasets spanning increasing dimensionality and different levels of radial mismatch. Our synthetic benchmarks contain $50{,}000$ samples each and include correlated Student-$t$ distributions with $\nu=3$ in dimensions $d=16$ and $d=32$, generated as
\[
X = ZA^\top,
\qquad
Z_i \overset{\mathrm{i.i.d.}}{\sim} \mathrm{Student}\text{-}t(\nu),
\]
where $A$ is a fixed random mixing matrix, as well as an anisotropic correlated Gaussian control. For real data, we use the same public planar PIV benchmark used by Gruhlke et al.~\cite{gruhlke2025multiplicative}, based on flow over a circular cylinder at Reynolds number $3900$ \cite{DHJXM6_2026}, and evaluate $d=64$ and $d=256$ vorticity representations. Full dataset construction and preprocessing details are deferred to Appendix~\ref{app:reproducibility}.

We compare \emph{Gaussian FM}, \emph{Source-only (empirical)}, \emph{RAFM (empirical)}, and \emph{Multiplicative Score Generative Models} (MSGM) \cite{gruhlke2025multiplicative}. On synthetic datasets, we additionally report \emph{Source-only (oracle)} and \emph{RAFM (oracle)} variants in Appendix~\ref{app:oracle_vs_empirical}. All methods use the same $3$-layer MLP with hidden width $128$. Unless otherwise stated, all models are trained with Adam for $10{,}000$ optimization steps using a common batch size of $256$, evaluated from $10{,}000$ generated samples, and averaged over three independent seeds. RAFM uses tangential projection at inference time. We report radial Wasserstein-$1$, the Kolmogorov--Smirnov (KS) statistic between generated and test radial CDFs, and Sliced Wasserstein-$1$ over $500$ random projections. Full implementation details are given in Appendix~\ref{app:reproducibility}.

\begin{table}[t]
\centering
\small
\setlength{\tabcolsep}{5pt}
\resizebox{\linewidth}{!}{
\begin{tabular}{ll|cccc}
\toprule
Dataset & Method & Radial W1 $\downarrow$ & KS $\downarrow$ & Sliced W1 $\downarrow$ & Train time \\
\midrule
\multirow{4}{*}{Student-$t$ ($d=16$)}
& Gaussian FM   & $3.3415 \scriptstyle{\pm 0.9074}$ & $0.1745 \scriptstyle{\pm 0.0627}$ & $0.7595 \scriptstyle{\pm 0.2315}$ & $18.3 \scriptstyle{\pm 0.2}$ s \\
& Source-only   & $0.3986 \scriptstyle{\pm 0.0444}$ & $0.0207 \scriptstyle{\pm 0.0055}$ & $0.4379 \scriptstyle{\pm 0.0644}$ & $19.0 \scriptstyle{\pm 0.3}$ s \\
& RAFM          & $\mathbf{0.2264} \scriptstyle{\pm 0.0476}$ & $\mathbf{0.0119} \scriptstyle{\pm 0.0029}$ & $\mathbf{0.3316} \scriptstyle{\pm 0.0298}$ & $35.2 \scriptstyle{\pm 0.1}$ s \\
& MSGM          & $0.3760 \scriptstyle{\pm 0.0650}$ & $0.0142 \scriptstyle{\pm 0.0036}$ & $0.4864 \scriptstyle{\pm 0.0279}$ & $49.3 \scriptstyle{\pm 0.3}$ min \\
\midrule
\multirow{4}{*}{Student-$t$ ($d=32$)}
& Gaussian FM   & $8.3012 \scriptstyle{\pm 0.5585}$ & $0.2960 \scriptstyle{\pm 0.0208}$ & $1.5186 \scriptstyle{\pm 0.0261}$ & $18.4 \scriptstyle{\pm 0.2}$ s \\
& Source-only   & $1.4295 \scriptstyle{\pm 0.2491}$ & $0.0369 \scriptstyle{\pm 0.0090}$ & $0.7513 \scriptstyle{\pm 0.2040}$ & $19.2 \scriptstyle{\pm 0.0}$ s \\
& RAFM          & $0.6162 \scriptstyle{\pm 0.0076}$ & $0.0120 \scriptstyle{\pm 0.0008}$ & $\mathbf{0.4749} \scriptstyle{\pm 0.0601}$ & $35.4 \scriptstyle{\pm 0.2}$ s \\
& MSGM          & $\mathbf{0.3747} \scriptstyle{\pm 0.0811}$ & $\mathbf{0.0112} \scriptstyle{\pm 0.0026}$ & $0.6812 \scriptstyle{\pm 0.0210}$ & $51.6 \scriptstyle{\pm 0.5}$ min \\
\midrule
\multirow{4}{*}{PIV ($d=64$)}
& Gaussian FM   & $0.3043 \scriptstyle{\pm 0.0218}$ & $0.2272 \scriptstyle{\pm 0.0032}$ & $0.0463 \scriptstyle{\pm 0.0030}$ & $18.3 \scriptstyle{\pm 0.0}$ s \\
& Source-only   & $0.1017 \scriptstyle{\pm 0.0032}$ & $0.1068 \scriptstyle{\pm 0.0056}$ & $0.0324 \scriptstyle{\pm 0.0046}$ & $19.3 \scriptstyle{\pm 0.0}$ s \\
& RAFM          & $0.0482 \scriptstyle{\pm 0.0019}$ & $\mathbf{0.0469} \scriptstyle{\pm 0.0026}$ & $\mathbf{0.0273} \scriptstyle{\pm 0.0029}$ & $35.8 \scriptstyle{\pm 0.2}$ s \\
& MSGM          & $\mathbf{0.0459} \scriptstyle{\pm 0.0037}$ & $0.0474 \scriptstyle{\pm 0.0039}$ & $0.0539 \scriptstyle{\pm 0.0003}$ & $46.9 \scriptstyle{\pm 0.4}$ min \\
\midrule
\multirow{4}{*}{PIV ($d=256$)}
& Gaussian FM   & $8.9522 \scriptstyle{\pm 0.0243}$ & $1.0000 \scriptstyle{\pm 0.0000}$ & $0.4498 \scriptstyle{\pm 0.0024}$ & $18.5 \scriptstyle{\pm 0.0}$ s \\
& Source-only   & $0.5696 \scriptstyle{\pm 0.0164}$ & $0.3971 \scriptstyle{\pm 0.0123}$ & $0.0382 \scriptstyle{\pm 0.0019}$ & $19.4 \scriptstyle{\pm 0.1}$ s \\
& RAFM          & $\mathbf{0.0371} \scriptstyle{\pm 0.0013}$ & $\mathbf{0.0579} \scriptstyle{\pm 0.0014}$ & $\mathbf{0.0242} \scriptstyle{\pm 0.0030}$ & $35.4 \scriptstyle{\pm 0.1}$ s \\
& MSGM          & $0.0429 \scriptstyle{\pm 0.0030}$ & $0.0665 \scriptstyle{\pm 0.0023}$ & $0.0498 \scriptstyle{\pm 0.0015}$ & $5.39 \scriptstyle{\pm 0.02}$ h \\
\bottomrule
\end{tabular}
}
\caption{Main benchmark results under a harmonized setting with a common batch size of 256 and 10,000 training steps for all methods. RAFM substantially improves over Gaussian FM and source-only baselines. Compared with MSGM, RAFM remains competitive across the main regimes while being dramatically faster in wall-clock time, and achieves its strongest result on PIV ($d=256$), where it outperforms MSGM on all reported metrics.}
\label{tab:main_results}
\end{table}

\subsection{Main results}

Figure~\ref{fig:radial_tail_diagnostics} visualizes this mechanism on two representative hard regimes. The top row shows the source mismatch: in both Student-$t$ ($d=32$) and PIV ($d=256$), the Gaussian radial law is poorly aligned with the test radial distribution, whereas the empirical radial source closely matches it. The bottom row shows the resulting generated radial fidelity. Gaussian FM inherits this mismatch, source-only already recovers a large part of the gap, and RAFM further improves the match on the hardest settings. On Student-$t$ ($d=32$), MSGM remains slightly stronger on radial fidelity, whereas RAFM achieves lower Sliced Wasserstein while training roughly two orders of magnitude faster. On PIV ($d=256$), RAFM outperforms MSGM on all three reported metrics while remaining substantially cheaper to train.

Table~\ref{tab:main_results} reports the main quantitative benchmarks: two heavy-tailed Student-$t$ settings and two real PIV settings. Several trends are consistent across datasets. First, standard Gaussian FM degrades sharply when radial mismatch is substantial, especially in the heavier-tailed and higher-dimensional regimes. Second, replacing only the source already yields large gains, confirming that source mismatch is a major part of the problem. Third, full RAFM further improves over source-only on the hardest regimes, showing that once the radial law is corrected, geometry-aware transport also matters. This is consistent with the theoretical picture developed in the previous section: correcting the radial mismatch removes the dominant source error, after which the path design becomes the next limiting factor.

Compared with MSGM, RAFM is consistently competitive while remaining dramatically cheaper to train under the harmonized 10k-step, batch-size-256 setting. On Student-$t$ ($d=16$), RAFM outperforms MSGM on all three main metrics. On Student-$t$ ($d=32$), the comparison is more nuanced: MSGM is slightly stronger on radial W1 and KS, whereas RAFM achieves substantially better Sliced Wasserstein. On PIV ($d=64$), the two methods are nearly tied on radial W1 and KS, while RAFM improves markedly on Sliced Wasserstein. On PIV ($d=256$), RAFM delivers the strongest overall result, outperforming MSGM on all three reported metrics while training in seconds rather than hours. As expected, on milder regimes such as the anisotropic Gaussian benchmark and low-dimensional PIV, the gains are smaller and source-only already captures most of the improvement. Full secondary-control results are reported in Appendix~\ref{app:experiments}.

\subsection{Projection and additional analyses}

Inference-time tangential projection is mainly a practical stabilizer for difficult regimes. It is not uniformly helpful on the easiest controls, but it becomes increasingly important when radial mismatch and dimensionality grow: removing it substantially degrades radial fidelity on Student-$t$ ($d=32$), PIV ($d=64$), and especially PIV ($d=256$). We report the full ablation in Appendix~\ref{app:experiments}, Table~\ref{tab:full_projection_ablation}. On synthetic datasets, oracle and empirical radial sources remain close, supporting the practical viability of estimating the radial law from training norms (Appendix~\ref{app:oracle_vs_empirical}). Finally, Appendix~\ref{app:toy2d} reports a two-dimensional radial--angular toy that exposes a genuine near-origin low-dimensional failure mode of the current spherical construction. Taken together, these additional analyses support the main conclusion of this section: matching the radial law is the dominant source of improvement, while the spherical path and tangential projection matter most in the hardest non-Gaussian regimes.
\section{Conclusion}

We revisited Flow Matching in the regime where the standard Gaussian source induces a structural radial mismatch with the data. For heavy-tailed or anisotropic distributions, this mismatch is not a minor modeling detail: it forces the transport to spend part of its capacity correcting an artificial discrepancy in norm statistics before modeling the structure that actually characterizes the target distribution.

RAFM addresses this issue by combining a source matched to the data radial law with conditional spherical paths that preserve radius and transport mass mainly through directions. This preserves the standard simulation-free Conditional Flow Matching pipeline while introducing a non-Gaussian inductive bias directly at the level of source and path design.

Our analysis shows that this construction removes the radial KL penalty associated with a Gaussian source, preserves the matched radial structure under tangential dynamics, and links target regression error to generation error through a Wasserstein stability bound. Empirically, the results confirm that correcting the source radial law is the dominant factor of improvement, while spherical transport provides additional gains once the radial mismatch has been removed, especially in the most challenging regimes. Across the main benchmarks, RAFM is consistently competitive with MSGM and often improves upon it, with particularly strong results on Student-$t$ ($d=16$) and PIV ($d=256$), while requiring substantially less training time than MSGM.

More broadly, these results suggest that the source distribution in Flow Matching should be treated as a geometric design choice rather than as a neutral default. When the data exhibit non-Gaussian radial structure, adapting the source and the transport jointly can lead to a better aligned and more efficient generative model. A current limitation of the proposed construction is its fragility in very low-dimensional near-origin regimes, which motivates future work on more robust angular transports and broader non-Gaussian path designs.

\newpage

{
\bibliographystyle{plain}
\bibliography{references}

@inproceedings{sohl2015deep,
  title     = {Deep Unsupervised Learning Using Nonequilibrium Thermodynamics},
  author    = {Sohl-Dickstein, Jascha and Weiss, Eric and Maheswaranathan, Niru and Ganguli, Surya},
  booktitle = {International Conference on Machine Learning (ICML)},
  pages     = {2256--2265},
  year      = {2015},
  organization = {PMLR}
}

@inproceedings{rezende2015variational,
  title     = {Variational Inference with Normalizing Flows},
  author    = {Rezende, Danilo and Mohamed, Shakir},
  booktitle = {International Conference on Machine Learning (ICML)},
  pages     = {1530--1538},
  year      = {2015},
  organization = {PMLR}
}

@article{dinh2016density,
  title   = {Density Estimation Using Real NVP},
  author  = {Dinh, Laurent and Sohl-Dickstein, Jascha and Bengio, Samy},
  journal = {arXiv preprint arXiv:1605.08803},
  year    = {2016}
}

@article{chen2018neural,
  title   = {Neural Ordinary Differential Equations},
  author  = {Chen, Ricky T. Q. and Rubanova, Yulia and Bettencourt, Jesse and Duvenaud, David K.},
  journal = {Advances in Neural Information Processing Systems},
  volume  = {31},
  year    = {2018}
}

@inproceedings{rezende2020normalizing,
  title        = {Normalizing Flows on Tori and Spheres},
  author       = {Rezende, Danilo Jimenez and Papamakarios, George and Racaniere, S{\'e}bastien
                  and Albergo, Michael and Kanwar, Gurtej and Shanahan, Phiala and Cranmer, Kyle},
  booktitle    = {International Conference on Machine Learning (ICML)},
  pages        = {8083--8092},
  year         = {2020},
  organization = {PMLR}
}

@inproceedings{ho2020ddpm,
  title     = {Denoising Diffusion Probabilistic Models},
  author    = {Ho, Jonathan and Jain, Ajay and Abbeel, Pieter},
  booktitle = {Advances in Neural Information Processing Systems (NeurIPS)},
  year      = {2020},
  url       = {https://proceedings.neurips.cc/paper/2020/hash/4c5bcfec8584af0d967f1ab10179ca4b-Abstract.html}
}

@inproceedings{song2021sde,
  title     = {Score-Based Generative Modeling through Stochastic Differential Equations},
  author    = {Song, Yang and Sohl{-}Dickstein, Jascha and Kingma, Durk P. and Kumar, Abhishek
               and Ermon, Stefano and Poole, Ben},
  booktitle = {International Conference on Learning Representations (ICLR)},
  year      = {2021},
  url       = {https://iclr.cc/virtual/2021/oral/3402}
}

@article{de2021diffusion,
  title   = {Diffusion {S}chr{\"o}dinger Bridge with Applications to Score-Based Generative Modeling},
  author  = {De Bortoli, Valentin and Thornton, James and Heng, Jeremy and Doucet, Arnaud},
  journal = {Advances in Neural Information Processing Systems},
  volume  = {34},
  pages   = {17695--17709},
  year    = {2021}
}

@article{karras2022elucidating,
  title   = {Elucidating the Design Space of Diffusion-Based Generative Models},
  author  = {Karras, Tero and Aittala, Miika and Aila, Timo and Laine, Samuli},
  journal = {Advances in Neural Information Processing Systems},
  volume  = {35},
  pages   = {26565--26577},
  year    = {2022}
}

@inproceedings{lipman2023flowmatching,
  title     = {Flow Matching for Generative Modeling},
  author    = {Lipman, Yaron and Chen, Ricky T. Q. and Ben{-}Hamu, Heli and Nickel, Maximilian and Le, Matthew},
  booktitle = {The Eleventh International Conference on Learning Representations (ICLR)},
  year      = {2023},
  publisher = {OpenReview.net},
  url       = {https://openreview.net/forum?id=PqvMRDCJT9t}
}

@inproceedings{liu2023rectifiedflow,
  title     = {Flow Straight and Fast: Learning to Generate and Transfer Data with Rectified Flow},
  author    = {Liu, Xingchao and Gong, Chengyue and Liu, Qiang},
  booktitle = {The Eleventh International Conference on Learning Representations (ICLR)},
  year      = {2023},
  publisher = {OpenReview.net},
  url       = {https://openreview.net/forum?id=XVjTT1nw5z}
}

@article{tong2023simulation,
  title   = {Simulation-Free {S}chr{\"o}dinger Bridges via Score and Flow Matching},
  author  = {Tong, Alexander and Malkin, Nikolay and Fatras, Kilian and Atanackovic, Lazar
             and Zhang, Yanlei and Huguet, Guillaume and Wolf, Guy and Bengio, Yoshua},
  journal = {arXiv preprint arXiv:2307.03672},
  year    = {2023}
}

@article{kerrigan2023functional,
  title   = {Functional Flow Matching},
  author  = {Kerrigan, Gavin and Migliorini, Giosue and Smyth, Padhraic},
  journal = {arXiv preprint arXiv:2305.17209},
  year    = {2023}
}

@article{kornilov2024optimal,
  title   = {Optimal Flow Matching: Learning Straight Trajectories in Just One Step},
  author  = {Kornilov, Nikita and Mokrov, Petr and Gasnikov, Alexander and Korotin, Alexander},
  journal = {Advances in Neural Information Processing Systems},
  volume  = {37},
  pages   = {104180--104204},
  year    = {2024}
}

@article{rozen2021moser,
  title   = {Moser Flow: Divergence-Based Generative Modeling on Manifolds},
  author  = {Rozen, Noam and Grover, Aditya and Nickel, Maximilian and Lipman, Yaron},
  journal = {Advances in Neural Information Processing Systems},
  volume  = {34},
  pages   = {17669--17680},
  year    = {2021}
}

@article{ben2022matching,
  title   = {Matching Normalizing Flows and Probability Paths on Manifolds},
  author  = {Ben-Hamu, Heli and Cohen, Samuel and Bose, Joey and Amos, Brandon and Grover, Aditya
             and Nickel, Maximilian and Chen, Ricky T. Q. and Lipman, Yaron},
  journal = {arXiv preprint arXiv:2207.04711},
  year    = {2022}
}

@article{de2022riemannian,
  title   = {Riemannian Score-Based Generative Modelling},
  author  = {De Bortoli, Valentin and Mathieu, Emile and Hutchinson, Michael and Thornton, James
             and Teh, Yee Whye and Doucet, Arnaud},
  journal = {Advances in Neural Information Processing Systems},
  volume  = {35},
  pages   = {2406--2422},
  year    = {2022}
}

@article{chen2023flow,
  title   = {Flow Matching on General Geometries},
  author  = {Chen, Ricky T. Q. and Lipman, Yaron},
  journal = {arXiv preprint arXiv:2302.03660},
  year    = {2023}
}

@article{kapusniak2024metric,
  title   = {Metric Flow Matching for Smooth Interpolations on the Data Manifold},
  author  = {Kapu{\'s}niak, Kacper and Potaptchik, Peter and Reu, Teodora and Zhang, Leo
             and Tong, Alexander and Bronstein, Michael and Bose, Avishek J. and Di Giovanni, Francesco},
  journal = {Advances in Neural Information Processing Systems},
  volume  = {37},
  pages   = {135011--135042},
  year    = {2024}
}

@inproceedings{jaini2020tails,
  title        = {Tails of {L}ipschitz Triangular Flows},
  author       = {Jaini, Priyank and Kobyzev, Ivan and Yu, Yaoliang and Brubaker, Marcus},
  booktitle    = {International Conference on Machine Learning (ICML)},
  pages        = {4673--4681},
  year         = {2020},
  organization = {PMLR}
}

@inproceedings{stimper2022resampling,
  title        = {Resampling Base Distributions of Normalizing Flows},
  author       = {Stimper, Vincent and Sch{\"o}lkopf, Bernhard and Hern{\'a}ndez-Lobato, Jos{\'e} Miguel},
  booktitle    = {International Conference on Artificial Intelligence and Statistics (AISTATS)},
  pages        = {4915--4936},
  year         = {2022},
  organization = {PMLR}
}

@inproceedings{laszkiewicz2022marginal,
  title        = {Marginal Tail-Adaptive Normalizing Flows},
  author       = {Laszkiewicz, Mike and Lederer, Johannes and Fischer, Asja},
  booktitle    = {International Conference on Machine Learning (ICML)},
  pages        = {12020--12048},
  year         = {2022},
  organization = {PMLR}
}

@article{hickling2024flexible,
  title   = {Flexible Tails for Normalizing Flows},
  author  = {Hickling, Tennessee and Prangle, Dennis},
  journal = {arXiv preprint arXiv:2406.16971},
  year    = {2024}
}

@article{pandey2024heavy,
  title   = {Heavy-Tailed Diffusion Models},
  author  = {Pandey, Kushagra and Pathak, Jaideep and Xu, Yilun and Mandt, Stephan
             and Pritchard, Michael and Vahdat, Arash and Mardani, Morteza},
  journal = {arXiv preprint arXiv:2410.14171},
  year    = {2024}
}

@inproceedings{gruhlke2025multiplicative,
  title     = {Multiplicative Diffusion Models: Beyond {G}aussian Latents},
  author    = {Gruhlke, Robert and Resseguier, Valentin and Makougne Merveille, Cyndie Talla},
  booktitle = {The Fourteenth International Conference on Learning Representations (ICLR)},
  year      = {2025}
}

@article{cont2001empirical,
  title     = {Empirical Properties of Asset Returns: Stylized Facts and Statistical Issues},
  author    = {Cont, Rama},
  journal   = {Quantitative Finance},
  volume    = {1},
  number    = {2},
  pages     = {223},
  year      = {2001},
  publisher = {IOP Publishing}
}

@article{papalexiou2013extreme,
  title     = {How Extreme is Extreme? An Assessment of Daily Rainfall Distribution Tails},
  author    = {Papalexiou, S. M. and Koutsoyiannis, D. and Makropoulos, C.},
  journal   = {Hydrology and Earth System Sciences},
  volume    = {17},
  number    = {2},
  pages     = {851--862},
  year      = {2013},
  publisher = {Copernicus Publications G{\"o}ttingen, Germany}
}

@data{DHJXM6_2026,
  title     = {{Non-time-resolved PIV Dataset of Flow over a Circular Cylinder at Reynolds Number 3900}},
  author    = {GEORGEAULT, Philippe and HEITZ, Dominique},
  publisher = {Recherche Data Gouv},
  year      = {2026},
  version   = {V1},
  doi       = {10.57745/DHJXM6},
  url       = {https://doi.org/10.57745/DHJXM6}
}

@article{vershynin2020high,
  title   = {High-Dimensional Probability},
  author  = {Vershynin, Roman},
  journal = {University of California, Irvine},
  volume  = {10},
  number  = {11},
  pages   = {31},
  year    = {2020}
}
}

\newpage
\appendix
\appendix

\section{Additional theory for RAFM}
\label{app:theory}

\subsection{Properties of the radial source}
\label{app:radial_source}

This appendix complements Section~\ref{subsec:radial_source}. We collect here the formal properties
of the radial source construction, its comparison with a Gaussian source, and the corresponding
empirical extension.

\paragraph{Polar notation.}
Let $X\sim p_{\mathrm{data}}$ be a random vector in $\mathbb{R}^d$, with $d\ge 2$, and assume
$p_{\mathrm{data}}(\{0\})=0$. We write
\[
R:=\|X\|\in \mathbb{R}_+,
\qquad
U:=\frac{X}{\|X\|}\in S^{d-1}.
\]
We denote by $p_R$ the density of $R$ on $\mathbb{R}_+$, by $\sigma$ the surface measure on
$S^{d-1}$, and by $p_{U\mid R}(\cdot\mid r)$ the conditional angular density with respect to $\sigma$.
Under polar coordinates $x=ru$, the Lebesgue measure decomposes as
\[
dx = r^{d-1}\,dr\,d\sigma(u).
\]

\begin{proposition}[Density of the radial source]
\label{app:prop:radial_source_density}
Let
\[
X_0 = R\,U_0,
\qquad
R\sim p_R,
\qquad
U_0\sim \mathrm{Unif}(S^{d-1}),
\qquad
R \perp U_0.
\]
Then the law $q_{\mathrm{rad}}$ of $X_0$ is absolutely continuous on $\mathbb{R}^d\setminus\{0\}$,
with density
\[
q_{\mathrm{rad}}(x)
=
\frac{p_R(\|x\|)}{|S^{d-1}|\,\|x\|^{d-1}},
\qquad x\neq 0.
\]
\end{proposition}

\begin{proof}
Let $f:\mathbb{R}^d\to\mathbb{R}$ be bounded measurable. Using polar coordinates and the independence
of $R$ and $U_0$,
\[
\mathbb{E}[f(X_0)]
=
\int_0^\infty \int_{S^{d-1}} f(ru)\,\frac{1}{|S^{d-1}|}\,d\sigma(u)\,p_R(r)\,dr.
\]
Since $dx=r^{d-1}\,dr\,d\sigma(u)$, this can be rewritten as
\[
\mathbb{E}[f(X_0)]
=
\int_{\mathbb{R}^d\setminus\{0\}}
f(x)\,
\frac{p_R(\|x\|)}{|S^{d-1}|\,\|x\|^{d-1}}
\,dx.
\]
Hence the density of $q_{\mathrm{rad}}$ is exactly the claimed expression.
\end{proof}

\begin{proposition}[Exact radial preservation]
\label{app:prop:radial_preservation}
Let $X\sim p_{\mathrm{data}}$ and $X_0\sim q_{\mathrm{rad}}$. Then
\[
\|X_0\| \overset{d}{=} \|X\|.
\]
In particular, for every $t\ge 0$,
\[
\mathbb{P}(\|X_0\|>t)=\mathbb{P}(\|X\|>t).
\]
\end{proposition}

\begin{proof}
By construction, $X_0=R\,U_0$ with $\|U_0\|=1$ almost surely, so
\[
\|X_0\| = R.
\]
Since $R$ is exactly the norm of a sample from $p_{\mathrm{data}}$, the claim follows.
\end{proof}

\begin{proposition}[Gaussian under-dispersion for regularly varying radial tails]
\label{app:prop:gaussian_tail_mismatch}
Let $Z\sim \mathcal{N}(0,I_d)$ and assume that the data radial tail is regularly varying:
\[
\mathbb{P}(R>t)=t^{-\alpha}L(t)
\qquad \text{as } t\to\infty,
\]
for some $\alpha>0$ and some slowly varying function $L$. Then
\[
\frac{\mathbb{P}(\|Z\|>t)}{\mathbb{P}(R>t)} \to 0
\qquad \text{as } t\to\infty.
\]
\end{proposition}

\begin{proof}
The Euclidean norm of a standard Gaussian has a $\chi_d$ distribution, whose tail satisfies
\[
\mathbb{P}(\|Z\|>t)\le C_d\, t^{d-2}e^{-t^2/2}
\]
for all sufficiently large $t$, for some constant $C_d>0$. Therefore
\[
\frac{\mathbb{P}(\|Z\|>t)}{\mathbb{P}(R>t)}
\le
C_d \,\frac{t^{d-2}e^{-t^2/2}}{t^{-\alpha}L(t)}
=
C_d\,\frac{t^{d-2+\alpha}e^{-t^2/2}}{L(t)}.
\]
Since $L$ is slowly varying, it grows sub-polynomially, while the factor $e^{-t^2/2}$ dominates any
polynomial. Hence the right-hand side converges to $0$.
\end{proof}

\begin{theorem}[KL decomposition for the ideal radial source]
\label{app:thm:radial_kl}
Assume that the relevant conditional densities exist and that the KL divergences below are finite.
Then
\[
\mathrm{KL}(p_{\mathrm{data}}\|q_{\mathrm{rad}})
=
\mathbb{E}_{r\sim p_R}
\Big[
\mathrm{KL}\bigl(p_{U\mid R}(\cdot\mid r)\,\|\,\mathrm{Unif}(S^{d-1})\bigr)
\Big],
\]
whereas
\[
\mathrm{KL}(p_{\mathrm{data}}\|\phi_d)
=
\mathrm{KL}(p_R\|p_{\chi_d})
+
\mathbb{E}_{r\sim p_R}
\Big[
\mathrm{KL}\bigl(p_{U\mid R}(\cdot\mid r)\,\|\,\mathrm{Unif}(S^{d-1})\bigr)
\Big].
\]
Consequently,
\[
\mathrm{KL}(p_{\mathrm{data}}\|q_{\mathrm{rad}})
\le
\mathrm{KL}(p_{\mathrm{data}}\|\phi_d),
\]
with equality if and only if $p_R=p_{\chi_d}$ almost everywhere.
\end{theorem}

\begin{proof}
Under polar coordinates $x=ru$, the data density may be written as
\[
p_{\mathrm{data}}(x)
=
\frac{p_R(r)\,p_{U\mid R}(u\mid r)}{r^{d-1}},
\qquad r>0,\ u\in S^{d-1},
\]
where $p_{U\mid R}(\cdot\mid r)$ is a density with respect to $\sigma$.
By Proposition~\ref{app:prop:radial_source_density},
\[
q_{\mathrm{rad}}(x)
=
\frac{p_R(r)}{|S^{d-1}|\,r^{d-1}}.
\]
Therefore
\[
\log \frac{p_{\mathrm{data}}(x)}{q_{\mathrm{rad}}(x)}
=
\log p_{U\mid R}(u\mid r) + \log |S^{d-1}|.
\]
Integrating against $p_{\mathrm{data}}(x)\,dx = p_R(r)\,p_{U\mid R}(u\mid r)\,dr\,d\sigma(u)$ gives
\[
\mathrm{KL}(p_{\mathrm{data}}\|q_{\mathrm{rad}})
=
\int_0^\infty p_R(r)
\int_{S^{d-1}}
p_{U\mid R}(u\mid r)
\log\!\Big(
p_{U\mid R}(u\mid r)\,|S^{d-1}|
\Big)
\,d\sigma(u)\,dr.
\]
Since the uniform density on the sphere is $|S^{d-1}|^{-1}$ with respect to $\sigma$, this is exactly
\[
\mathbb{E}_{r\sim p_R}
\Big[
\mathrm{KL}\bigl(p_{U\mid R}(\cdot\mid r)\,\|\,\mathrm{Unif}(S^{d-1})\bigr)
\Big].
\]

For the standard Gaussian,
\[
\phi_d(x)
=
\frac{p_{\chi_d}(r)}{|S^{d-1}|\,r^{d-1}},
\]
since the Gaussian is conditionally uniform in direction at fixed radius. Hence
\[
\log \frac{p_{\mathrm{data}}(x)}{\phi_d(x)}
=
\log \frac{p_R(r)}{p_{\chi_d}(r)}
+
\log p_{U\mid R}(u\mid r)
+
\log |S^{d-1}|.
\]
Integrating again yields
\[
\mathrm{KL}(p_{\mathrm{data}}\|\phi_d)
=
\int_0^\infty p_R(r)\log\frac{p_R(r)}{p_{\chi_d}(r)}\,dr
+
\mathbb{E}_{r\sim p_R}
\Big[
\mathrm{KL}\bigl(p_{U\mid R}(\cdot\mid r)\,\|\,\mathrm{Unif}(S^{d-1})\bigr)
\Big].
\]
The first term is exactly $\mathrm{KL}(p_R\|p_{\chi_d})$, which proves the decomposition.
The inequality and equality condition follow immediately.
\end{proof}

\begin{theorem}[KL decomposition for an empirical radial source]
\label{app:thm:empirical_radial_kl}
Let $\widehat p_R$ be a strictly positive density on the support of $p_R$, and define
\[
\widehat q_{\mathrm{rad}}(x)
=
\frac{\widehat p_R(\|x\|)}{|S^{d-1}|\,\|x\|^{d-1}},
\qquad x\neq 0.
\]
Assume that all KL divergences below are finite. Then
\[
\mathrm{KL}(p_{\mathrm{data}}\|\widehat q_{\mathrm{rad}})
=
\mathrm{KL}(p_R\|\widehat p_R)
+
\mathbb{E}_{r\sim p_R}
\Big[
\mathrm{KL}\bigl(p_{U\mid R}(\cdot\mid r)\,\|\,\mathrm{Unif}(S^{d-1})\bigr)
\Big].
\]
In particular,
\[
\mathrm{KL}(p_{\mathrm{data}}\|\widehat q_{\mathrm{rad}})
-
\mathrm{KL}(p_{\mathrm{data}}\|\phi_d)
=
\mathrm{KL}(p_R\|\widehat p_R)
-
\mathrm{KL}(p_R\|p_{\chi_d}).
\]
\end{theorem}

\begin{proof}
The proof is identical to that of Theorem~\ref{app:thm:radial_kl}, replacing $p_R$ by
$\widehat p_R$ in the source density:
\[
\log \frac{p_{\mathrm{data}}(x)}{\widehat q_{\mathrm{rad}}(x)}
=
\log \frac{p_R(r)}{\widehat p_R(r)}
+
\log p_{U\mid R}(u\mid r)
+
\log |S^{d-1}|.
\]
Integrating against the polar factorization of $p_{\mathrm{data}}$ yields the result.
\end{proof}

\begin{corollary}[Asymptotic recovery of the ideal radial source]
\label{app:cor:empirical_radial_kl_limit}
Assume that $\widehat p_R$ is a sequence of strictly positive density estimators such that
\[
\mathrm{KL}(p_R\|\widehat p_R)\to 0.
\]
Then
\[
\mathrm{KL}(p_{\mathrm{data}}\|\widehat q_{\mathrm{rad}})
\to
\mathrm{KL}(p_{\mathrm{data}}\|q_{\mathrm{rad}}).
\]
\end{corollary}

\begin{proof}
Subtract the identity of Theorem~\ref{app:thm:radial_kl} from that of
Theorem~\ref{app:thm:empirical_radial_kl}. The angular term cancels, giving
\[
\mathrm{KL}(p_{\mathrm{data}}\|\widehat q_{\mathrm{rad}})
-
\mathrm{KL}(p_{\mathrm{data}}\|q_{\mathrm{rad}})
=
\mathrm{KL}(p_R\|\widehat p_R),
\]
which converges to $0$ by assumption.
\end{proof}

\paragraph{Remark on implementation.}
The KL statements above require a strictly positive radial density estimator. In practice, however,
RAFM can be initialized directly from an empirical CDF or by resampling the observed training
radii, which is often more natural in one dimension. In that case, Wasserstein or CDF-based error
metrics are more appropriate than KL divergence.

\subsection{Statistical properties of the empirical radial source}
\label{app:empirical_radial_source}

This appendix quantifies the approximation error introduced when the radial law is estimated
empirically from training data.

\paragraph{Empirical radial law.}
Let $X_1,\dots,X_n \sim p_{\mathrm{data}}$ be i.i.d. training samples, and define the corresponding radii
\[
R_i := \|X_i\|,
\qquad i=1,\dots,n.
\]
Let $\mu_R$ denote the law of $R=\|X\|$ for $X\sim p_{\mathrm{data}}$, and let
\[
F_R(r) := \mathbb P(R\le r),
\qquad r\ge 0,
\]
be its cumulative distribution function. The empirical radial measure is
\[
\widehat\mu_{R,n} := \frac1n\sum_{i=1}^n \delta_{R_i},
\]
with empirical CDF
\[
\widehat F_{R,n}(r) := \frac1n\sum_{i=1}^n \mathbf 1_{\{R_i\le r\}}.
\]

We define the empirical radial source by
\[
\widehat X_0 := \widehat R\,U_0,
\qquad
\widehat R \sim \widehat\mu_{R,n},
\qquad
U_0 \sim \mathrm{Unif}(S^{d-1}),
\qquad
\widehat R \perp U_0,
\]
and denote its law by $\widehat q_{\mathrm{rad},n}$.

\begin{proposition}[Consistency of the empirical radial CDF]
\label{app:prop:radial_cdf_consistency}
The empirical radial CDF converges uniformly almost surely:
\[
\sup_{r\ge 0}\big|\widehat F_{R,n}(r)-F_R(r)\big|\xrightarrow[n\to\infty]{a.s.}0.
\]
\end{proposition}

\begin{proof}
This is the Glivenko--Cantelli theorem applied to the one-dimensional sample $(R_i)_{i=1}^n$.
\end{proof}

\begin{theorem}[Dvoretzky--Kiefer--Wolfowitz bound]
\label{app:thm:radial_dkw}
For every $\varepsilon>0$,
\[
\mathbb P\!\left(
\sup_{r\ge 0}\big|\widehat F_{R,n}(r)-F_R(r)\big|>\varepsilon
\right)
\le
2e^{-2n\varepsilon^2}.
\]
Equivalently, for every $\delta\in(0,1)$, with probability at least $1-\delta$,
\[
\sup_{r\ge 0}\big|\widehat F_{R,n}(r)-F_R(r)\big|
\le
\sqrt{\frac{1}{2n}\log\frac{2}{\delta}}.
\]
\end{theorem}

\begin{proof}
This is the classical Dvoretzky--Kiefer--Wolfowitz inequality applied to the i.i.d. sample
$(R_i)_{i=1}^n$.
\end{proof}

\begin{proposition}[Transfer from radial estimation error to source estimation error]
\label{app:prop:radial_to_source_transfer}
Let $p\ge 1$. Then
\[
W_p(\widehat q_{\mathrm{rad},n},q_{\mathrm{rad}})
\le
W_p(\widehat\mu_{R,n},\mu_R).
\]
In particular, any convergence of the empirical radial law in Wasserstein distance induces the same
convergence for the corresponding radial source.
\end{proposition}

\begin{proof}
Let $\pi$ be any coupling between $\widehat\mu_{R,n}$ and $\mu_R$, and let
$U\sim \mathrm{Unif}(S^{d-1})$ be independent of $(\widehat R,R)\sim\pi$.
Then $(\widehat R U,RU)$ is a coupling between $\widehat q_{\mathrm{rad},n}$ and $q_{\mathrm{rad}}$.
Hence
\[
W_p^p(\widehat q_{\mathrm{rad},n},q_{\mathrm{rad}})
\le
\mathbb E\big[\|\widehat R U-RU\|^p\big]
=
\mathbb E\big[|\widehat R-R|^p\|U\|^p\big]
=
\mathbb E\big[|\widehat R-R|^p\big].
\]
Taking the infimum over all couplings $\pi$ yields the result.
\end{proof}

\begin{corollary}[High-probability control under bounded support]
\label{app:cor:bounded_support_radial_source}
Assume that the radial law is bounded:
\[
\mathbb P(R\le R_{\max})=1
\]
for some $R_{\max}>0$. Then
\[
W_1(\widehat q_{\mathrm{rad},n},q_{\mathrm{rad}})
\le
W_1(\widehat\mu_{R,n},\mu_R)
=
\int_0^{R_{\max}} |\widehat F_{R,n}(r)-F_R(r)|\,dr
\le
R_{\max}\sup_{r\ge 0}|\widehat F_{R,n}(r)-F_R(r)|.
\]
Consequently, for every $\delta\in(0,1)$, with probability at least $1-\delta$,
\[
W_1(\widehat q_{\mathrm{rad},n},q_{\mathrm{rad}})
\le
R_{\max}\sqrt{\frac{1}{2n}\log\frac{2}{\delta}}.
\]
\end{corollary}

\begin{proof}
For one-dimensional distributions supported on $[0,R_{\max}]$,
\[
W_1(\widehat\mu_{R,n},\mu_R)=\int_0^{R_{\max}} |\widehat F_{R,n}(r)-F_R(r)|\,dr.
\]
Therefore
\[
W_1(\widehat\mu_{R,n},\mu_R)
\le
R_{\max}\sup_{r\ge 0}|\widehat F_{R,n}(r)-F_R(r)|.
\]
Combining this with Proposition~\ref{app:prop:radial_to_source_transfer} and
Theorem~\ref{app:thm:radial_dkw} yields the claim.
\end{proof}

\paragraph{Interpretation.}
The previous results justify the use of an empirical radial source in practice. The empirical radial law
is estimated uniformly from one-dimensional training radii, and the resulting estimation error transfers
directly to the full radial source. Thus, $\widehat q_{\mathrm{rad},n}$ converges to the ideal source
$q_{\mathrm{rad}}$ as the sample size increases.

\subsection{Geometric properties of the spherical path}
\label{app:spherical_path}

This appendix complements Section~\ref{subsec:spherical_paths}. We state the geometric properties
of the spherical path, give the derivative formula, and provide an explicit conditional vector field.

\paragraph{Definition of the path.}
Let $x_0=Ru_0$ and $x_1=Ru_1$ with $R>0$ and $u_0,u_1\in S^{d-1}$. For non-antipodal pairs
$u_0\neq -u_1$, define
\[
\theta := \arccos(\langle u_0,u_1\rangle)\in[0,\pi),
\]
\[
\gamma_t(u_0,u_1)
=
\frac{\sin((1-t)\theta)}{\sin\theta}\,u_0
+
\frac{\sin(t\theta)}{\sin\theta}\,u_1,
\qquad t\in[0,1],
\]
and
\[
\psi_t(x_0,x_1):=R\,\gamma_t(u_0,u_1).
\]

\paragraph{Antipodal completion.}
When $u_0=-u_1$, the minimizing geodesic is not unique. Since $u_0$ is sampled uniformly on the
sphere conditionally on $x_1$, this event has conditional probability zero. Any fixed deterministic rule may be used on the antipodal set without affecting the construction in practice, since this event has probability zero under the sampling scheme.

\begin{proposition}[Geometric properties of the spherical path]
\label{app:prop:spherical_path_properties}
For any non-antipodal pair $(x_0,x_1)$ with $\|x_0\|=\|x_1\|=R$, the path
\[
X_t:=\psi_t(x_0,x_1)
\]
satisfies
\[
X_0=x_0,
\qquad
X_1=x_1,
\qquad
\|X_t\|=R
\quad \forall t\in[0,1].
\]
Moreover, its velocity is tangent to the scaled sphere $R\,S^{d-1}$:
\[
X_t^\top \dot X_t = 0
\qquad \forall t\in[0,1].
\]
\end{proposition}

\begin{proof}
The endpoint conditions follow from $\gamma_0(u_0,u_1)=u_0$ and $\gamma_1(u_0,u_1)=u_1$.
Since $\gamma_t(u_0,u_1)\in S^{d-1}$ for every $t$, we have
\[
\|X_t\| = R\|\gamma_t(u_0,u_1)\| = R.
\]
Differentiating $\|X_t\|^2=R^2$ gives
\[
2X_t^\top \dot X_t = 0,
\]
hence $X_t^\top \dot X_t=0$.
\end{proof}

\begin{proposition}[Derivative of the spherical interpolation]
\label{app:prop:path_derivative}
For any non-antipodal pair $(u_0,u_1)$,
\[
\partial_t \gamma_t(u_0,u_1)
=
\frac{\theta}{\sin\theta}
\left[
-\cos((1-t)\theta)\,u_0
+
\cos(t\theta)\,u_1
\right].
\]
Consequently,
\[
\partial_t \psi_t(x_0,x_1)
=
R\frac{\theta}{\sin\theta}
\left[
-\cos((1-t)\theta)\,u_0
+
\cos(t\theta)\,u_1
\right].
\]
\end{proposition}

\begin{proof}
Differentiate the coefficients of $\gamma_t(u_0,u_1)$ with respect to $t$:
\[
\frac{d}{dt}\frac{\sin((1-t)\theta)}{\sin\theta}
=
-\frac{\theta\cos((1-t)\theta)}{\sin\theta},
\qquad
\frac{d}{dt}\frac{\sin(t\theta)}{\sin\theta}
=
\frac{\theta\cos(t\theta)}{\sin\theta}.
\]
Substituting into the definition of $\gamma_t$ yields the first identity, and multiplying by $R$
gives the second.
\end{proof}

\paragraph{A closed-form conditional vector field.}
For $x,y\in R\,S^{d-1}$, define the geodesic angle
\[
\varphi_R(x,y)
:=
\arccos\!\left(\frac{\langle x,y\rangle}{R^2}\right)\in[0,\pi].
\]
The Riemannian logarithm map on the scaled sphere $R\,S^{d-1}$ is
\[
\mathrm{Log}^{(R)}_x(y)
=
\frac{\varphi_R(x,y)}{\sin(\varphi_R(x,y))}
\left(
y-\frac{\langle x,y\rangle}{R^2}x
\right),
\qquad x\neq -y.
\]
Using $R=\|x_1\|$, define
\[
u_t(x\mid x_1)
=
\frac{1}{1-t}\,\mathrm{Log}^{(R)}_x(x_1),
\qquad
t\in[0,1),\ x\in R\,S^{d-1},\ x\neq -x_1.
\]
Equivalently,
\[
u_t(x\mid x_1)
=
\frac{\varphi_R(x,x_1)}{(1-t)\sin(\varphi_R(x,x_1))}
\left(
x_1-\frac{\langle x,x_1\rangle}{R^2}x
\right).
\]
By construction,
\[
x^\top u_t(x\mid x_1)=0.
\]

\begin{proposition}[Consistency with the spherical path]
\label{app:prop:conditional_vf_consistency}
For any non-antipodal pair $(x_0,x_1)$ and any $t\in[0,1)$,
\[
u_t\!\bigl(\psi_t(x_0,x_1)\mid x_1\bigr)
=
\partial_t \psi_t(x_0,x_1).
\]
Hence $u_t(\cdot\mid x_1)$ generates the conditional path associated with $x_1$.
\end{proposition}

\begin{proof}
Let $X_t=\psi_t(x_0,x_1)=R\gamma_t(u_0,u_1)$. Its remaining geodesic angle to $x_1$ is
\[
\varphi_R(X_t,x_1)
=
\arccos\!\left(\frac{\langle X_t,x_1\rangle}{R^2}\right)
=
(1-t)\theta.
\]
Substituting this identity into the explicit formula for $u_t(x\mid x_1)$ yields exactly the expression
of Proposition~\ref{app:prop:path_derivative}.
\end{proof}

\subsection{Tangential dynamics and generation stability}
\label{app:stability}

This appendix collects the results summarized in Section~\ref{subsec:theory_summary}. We study the
learned flow induced by $v_\theta$ and relate target approximation to generation accuracy.

\paragraph{Learned dynamics.}
Let $v_\theta:[0,1]\times\mathbb{R}^d\to\mathbb{R}^d$ be a learned time-dependent vector field and
consider the ODE
\[
\dot Y_t = v_\theta(t,Y_t),
\qquad
Y_0\sim q_{\mathrm{rad}}.
\]
Whenever the ODE is well posed, we denote by $\Phi_t^\theta$ the associated flow map, so that
\[
Y_t=\Phi_t^\theta(Y_0),
\qquad
\mathrm{Law}(Y_t)=(\Phi_t^\theta)_\# q_{\mathrm{rad}}.
\]

\begin{assumption}[Regularity of the learned field]
\label{app:ass:learned_field_regular}
The learned vector field $v_\theta$ is Borel measurable in $t$, globally Lipschitz in $x$ uniformly
in $t$, and has at most linear growth: there exist constants $L_\theta,M_\theta\ge 0$ such that
for all $t\in[0,1]$ and all $x,y\in\mathbb{R}^d$,
\[
\|v_\theta(t,x)-v_\theta(t,y)\| \le L_\theta \|x-y\|,
\]
and
\[
\|v_\theta(t,x)\| \le M_\theta(1+\|x\|).
\]
\end{assumption}

\begin{proposition}[Well-posedness of the learned flow]
\label{app:prop:learned_flow_wellposed}
Under Assumption~\ref{app:ass:learned_field_regular}, for every initial condition $y_0\in\mathbb{R}^d$,
the ODE
\[
\dot Y_t = v_\theta(t,Y_t),
\qquad
Y_0=y_0,
\]
admits a unique absolutely continuous solution on $[0,1]$.
Consequently, the flow map $\Phi_t^\theta$ is well defined for every $t\in[0,1]$.
\end{proposition}

\begin{proof}
This is the standard Cauchy--Lipschitz theorem for time-dependent vector fields with at most linear
growth.
\end{proof}

\paragraph{Radial--tangential decomposition.}
For $x\neq 0$, any vector field can be decomposed uniquely as
\[
v_\theta(t,x)=\alpha_\theta(t,x)\,x + w_\theta(t,x),
\qquad
x^\top w_\theta(t,x)=0,
\]
where
\[
\alpha_\theta(t,x):=\frac{x^\top v_\theta(t,x)}{\|x\|^2},
\qquad
w_\theta(t,x):=v_\theta(t,x)-\alpha_\theta(t,x)x.
\]
The scalar $\alpha_\theta$ is the radial component, while $w_\theta$ is the tangential component.

\begin{proposition}[Purely tangential dynamics preserve the norm]
\label{app:prop:tangential_preserves_norm}
Assume that
\[
x^\top v_\theta(t,x)=0
\qquad
\forall (t,x)\in [0,1]\times (\mathbb{R}^d\setminus\{0\}).
\]
Let $Y_t$ solve
\[
\dot Y_t=v_\theta(t,Y_t)
\]
with $Y_0\neq 0$. Then
\[
\|Y_t\|=\|Y_0\|
\qquad \forall t\in[0,1].
\]
In particular, if $Y_0\sim q_{\mathrm{rad}}$, then
\[
\|Y_t\|\overset{d}{=}\|Y_0\|
\overset{d}{=}\|X\|,
\qquad X\sim p_{\mathrm{data}}.
\]
\end{proposition}

\begin{proof}
Differentiate $\frac12\|Y_t\|^2$:
\[
\frac{d}{dt}\frac12\|Y_t\|^2
=
Y_t^\top \dot Y_t
=
Y_t^\top v_\theta(t,Y_t)
=
0.
\]
Hence $\|Y_t\|^2=\|Y_0\|^2$ for all $t$.
\end{proof}

\begin{proposition}[Norm evolution under a radial component]
\label{app:prop:radial_norm_evolution}
Let $Y_t$ solve
\[
\dot Y_t = v_\theta(t,Y_t) = \alpha_\theta(t,Y_t)Y_t + w_\theta(t,Y_t),
\qquad
Y_0\neq 0,
\]
and consider any time interval on which $Y_t\neq 0$. Then
\[
\frac{d}{dt}\frac12\|Y_t\|^2
=
\alpha_\theta(t,Y_t)\|Y_t\|^2.
\]
Equivalently,
\[
\frac{d}{dt}\|Y_t\|
=
\alpha_\theta(t,Y_t)\|Y_t\|.
\]
Hence
\[
\|Y_t\|
=
\|Y_0\|
\exp\!\left(\int_0^t \alpha_\theta(s,Y_s)\,ds\right).
\]
\end{proposition}

\begin{proof}
Using the decomposition and the orthogonality condition,
\[
\frac{d}{dt}\frac12\|Y_t\|^2
=
Y_t^\top v_\theta(t,Y_t)
=
Y_t^\top \bigl(\alpha_\theta(t,Y_t)Y_t + w_\theta(t,Y_t)\bigr)
=
\alpha_\theta(t,Y_t)\|Y_t\|^2.
\]
The remaining identities follow immediately.
\end{proof}

\paragraph{Reference path and population regression error.}
Recall the matched-radius coupling of RAFM:
\[
X_1\sim p_{\mathrm{data}},
\qquad
X_0=\|X_1\|U_0,
\qquad
U_0\sim \mathrm{Unif}(S^{d-1}),
\qquad
X_t=\psi_t(X_0,X_1).
\]
We define the population RAFM regression error by
\[
\mathcal E_{\mathrm{RAFM}}(\theta)
:=
\mathbb E\int_0^1
\bigl\|
v_\theta(t,X_t)-\partial_t\psi_t(X_0,X_1)
\bigr\|_2^2\,dt.
\]

\begin{theorem}[Generation stability from target approximation]
\label{app:thm:generation_stability}
Assume that $\mathbb E[\|X_1\|^2]<\infty$ and that
Assumption~\ref{app:ass:learned_field_regular} holds with Lipschitz constant $L_\theta$.
Let $Y_t$ be the learned trajectory driven by $v_\theta$ from the same initial condition $X_0$:
\[
\dot Y_t=v_\theta(t,Y_t),
\qquad
Y_0=X_0.
\]
Then
\[
\mathbb E\|Y_1-X_1\|^2
\le
e^{2L_\theta}\,\mathcal E_{\mathrm{RAFM}}(\theta).
\]
Consequently,
\[
W_2^2\!\left((\Phi_1^\theta)_\# q_{\mathrm{rad}},\,p_{\mathrm{data}}\right)
\le
e^{2L_\theta}\,\mathcal E_{\mathrm{RAFM}}(\theta),
\]
and therefore
\[
W_2\!\left((\Phi_1^\theta)_\# q_{\mathrm{rad}},\,p_{\mathrm{data}}\right)
\le
e^{L_\theta}\,\mathcal E_{\mathrm{RAFM}}(\theta)^{1/2}.
\]
\end{theorem}

\begin{proof}
Let
\[
\Delta_t:=Y_t-X_t.
\]
Since $\dot X_t=\partial_t\psi_t(X_0,X_1)$, we have
\[
\dot\Delta_t
=
v_\theta(t,Y_t)-v_\theta(t,X_t)
+
\bigl(v_\theta(t,X_t)-\dot X_t\bigr).
\]
Taking norms and using the Lipschitz property of $v_\theta$ in space,
\[
\frac{d}{dt}\|\Delta_t\|
\le
L_\theta\|\Delta_t\|
+
\|v_\theta(t,X_t)-\dot X_t\|
\]
for almost every $t\in[0,1]$. Since $\Delta_0=0$, Gr\"onwall's inequality yields
\[
\|\Delta_1\|
\le
\int_0^1 e^{L_\theta(1-s)}
\|v_\theta(s,X_s)-\dot X_s\|\,ds
\le
e^{L_\theta}\int_0^1
\|v_\theta(s,X_s)-\dot X_s\|\,ds.
\]
Squaring and using Jensen's inequality,
\[
\|\Delta_1\|^2
\le
e^{2L_\theta}
\left(\int_0^1 \|v_\theta(s,X_s)-\dot X_s\|\,ds\right)^2
\le
e^{2L_\theta}
\int_0^1 \|v_\theta(s,X_s)-\dot X_s\|^2\,ds.
\]
Taking expectations proves
\[
\mathbb E\|Y_1-X_1\|^2
\le
e^{2L_\theta}\,\mathcal E_{\mathrm{RAFM}}(\theta).
\]

Since $Y_0=X_0\sim q_{\mathrm{rad}}$ and $Y_1=\Phi_1^\theta(Y_0)$,
\[
\mathrm{Law}(Y_1)=(\Phi_1^\theta)_\# q_{\mathrm{rad}}.
\]
Moreover, $X_1\sim p_{\mathrm{data}}$. Therefore the joint law of $(Y_1,X_1)$ is a coupling between
$(\Phi_1^\theta)_\# q_{\mathrm{rad}}$ and $p_{\mathrm{data}}$, so by the definition of Wasserstein distance,
\[
W_2^2\!\left((\Phi_1^\theta)_\# q_{\mathrm{rad}},\,p_{\mathrm{data}}\right)
\le
\mathbb E\|Y_1-X_1\|^2.
\]
Combining both inequalities concludes the proof.
\end{proof}

\begin{corollary}[Consistency under vanishing regression error]
\label{app:cor:rafm_consistency}
Under the assumptions of Theorem~\ref{app:thm:generation_stability}, if
\[
\mathcal E_{\mathrm{RAFM}}(\theta)\to 0,
\]
then
\[
W_2\!\left((\Phi_1^\theta)_\# q_{\mathrm{rad}},\,p_{\mathrm{data}}\right)\to 0.
\]
\end{corollary}

\begin{proof}
This follows immediately from Theorem~\ref{app:thm:generation_stability}.
\end{proof}

\section{Additional experimental results}
\label{app:experiments}

\subsection{Oracle versus empirical radial source}
\label{app:oracle_vs_empirical}

Table~\ref{tab:oracle_appendix} compares empirical and oracle variants of the radial source on the
synthetic Student-$t$ benchmarks. The empirical version remains close to the oracle one across
metrics, supporting the practical viability of estimating the radial law from training data.

\begin{table*}[t]
\centering
\small
\setlength{\tabcolsep}{4pt}
\resizebox{\linewidth}{!}{
\begin{tabular}{l|cc|cc}
\toprule
& \multicolumn{2}{c|}{Source-only} & \multicolumn{2}{c}{RAFM} \\
Dataset
& Empirical & Oracle
& Empirical & Oracle \\
\midrule
Student-$t$ ($d=16$), Radial W1
& $0.3986 \pm 0.0444$ & $0.5083 \pm 0.1319$
& $0.2264 \pm 0.0476$ & $0.2377 \pm 0.0562$ \\
Student-$t$ ($d=16$), KS
& $0.0207 \pm 0.0055$ & $0.0271 \pm 0.0048$
& $0.0119 \pm 0.0029$ & $0.0133 \pm 0.0037$ \\
Student-$t$ ($d=16$), Sliced W1
& $0.4379 \pm 0.0644$ & $0.4332 \pm 0.0595$
& $0.3316 \pm 0.0298$ & $0.3195 \pm 0.0242$ \\
\midrule
Student-$t$ ($d=32$), Radial W1
& $1.4295 \pm 0.2491$ & $1.4800 \pm 0.4106$
& $0.6162 \pm 0.0076$ & $0.4652 \pm 0.0185$ \\
Student-$t$ ($d=32$), KS
& $0.0369 \pm 0.0090$ & $0.0376 \pm 0.0076$
& $0.0120 \pm 0.0008$ & $0.0110 \pm 0.0021$ \\
Student-$t$ ($d=32$), Sliced W1
& $0.7513 \pm 0.2040$ & $0.7614 \pm 0.2074$
& $0.4749 \pm 0.0601$ & $0.4672 \pm 0.0494$ \\
\bottomrule
\end{tabular}
}
\caption{Empirical versus oracle radial distributions on the synthetic Student-$t$ benchmarks. The empirical variants remain of the same order as the oracle ones, supporting the practical viability of radial estimation from training data.}
\label{tab:oracle_appendix}
\end{table*}

\subsection{Two-dimensional radial--angular toy: failure mode}
\label{app:toy2d}

We defer the two-dimensional radial--angular toy experiment to the appendix. This dataset is visually
intuitive and combines a heavy-tailed radial law with a multimodal angular structure, but it also reveals
a limitation of RAFM in very low dimension. In $d=2$, the sphere reduces to the circle, leaving only
one angular degree of freedom. As a consequence, spherical geodesic paths are highly constrained,
trajectory crossings become more problematic, and the tangential projection becomes numerically
fragile near the origin when $\|x\|\approx 0$. We therefore view this experiment as a failure-mode
analysis rather than as a representative benchmark for the higher-dimensional setting targeted in the
main paper.

\begin{table}[t]
\centering
\small
\setlength{\tabcolsep}{5pt}
\begin{tabular}{lcccc}
\toprule
Method & Radial W1 $\downarrow$ & KS $\downarrow$ & Sliced W1 $\downarrow$ & Angular SW $\downarrow$ \\
\midrule
Gaussian FM & $0.0821 \pm 0.0426$ & $0.0563 \pm 0.0068$ & $0.0800 \pm 0.0134$ & $0.0502 \pm 0.0051$ \\
Source-only & $0.0688 \pm 0.0245$ & $0.0330 \pm 0.0090$ & $0.0889 \pm 0.0241$ & $0.0980 \pm 0.0375$ \\
RAFM & $0.0233 \pm 0.0049$ & $0.0291 \pm 0.0292$ & $0.3371 \pm 0.0421$ & $0.4695 \pm 0.0988$ \\
MSGM & $0.0488 \pm 0.0087$ & $0.0203 \pm 0.0026$ & $0.0449 \pm 0.0038$ & $0.0347 \pm 0.0011$ \\
\bottomrule
\end{tabular}
\caption{Failure mode on the 2D radial--angular toy. RAFM still improves radial fidelity, but its current spherical angular transport remains clearly weaker on global and angular geometry in this low-dimensional near-origin regime.}
\label{tab:toy_failure_mode}
\end{table}

\begin{table*}[t]
\centering
\small
\setlength{\tabcolsep}{5pt}
\begin{tabular}{ll|ccc}
\toprule
Dataset & RAFM variant & Radial W1 $\downarrow$ & KS $\downarrow$ & Sliced W1 $\downarrow$ \\
\midrule
\multirow{2}{*}{Gaussian aniso. ($d=16$)}
& w/ tangential projection  & $\mathbf{0.0607} \scriptstyle{\pm 0.0030}$ & $\mathbf{0.0127} \scriptstyle{\pm 0.0011}$ & $0.1568 \scriptstyle{\pm 0.0107}$ \\
& w/o tangential projection & $0.1909 \scriptstyle{\pm 0.0355}$ & $0.0200 \scriptstyle{\pm 0.0026}$ & $\mathbf{0.1444} \scriptstyle{\pm 0.0150}$ \\
\midrule
\multirow{2}{*}{PIV ($d=16$)}
& w/ tangential projection  & $\mathbf{0.0111} \scriptstyle{\pm 0.0005}$ & $\mathbf{0.0824} \scriptstyle{\pm 0.0050}$ & $0.0114 \scriptstyle{\pm 0.0011}$ \\
& w/o tangential projection & $0.0164 \scriptstyle{\pm 0.0013}$ & $0.0848 \scriptstyle{\pm 0.0156}$ & $\mathbf{0.0113} \scriptstyle{\pm 0.0017}$ \\
\midrule
\multirow{2}{*}{PIV ($d=64$)}
& w/ tangential projection  & $\mathbf{0.0482} \scriptstyle{\pm 0.0019}$ & $\mathbf{0.0469} \scriptstyle{\pm 0.0026}$ & $\mathbf{0.0273} \scriptstyle{\pm 0.0029}$ \\
& w/o tangential projection & $0.1656 \scriptstyle{\pm 0.0119}$ & $0.1304 \scriptstyle{\pm 0.0177}$ & $0.0289 \scriptstyle{\pm 0.0016}$ \\
\midrule
\multirow{2}{*}{PIV ($d=256$)}
& w/ tangential projection  & $\mathbf{0.0371} \scriptstyle{\pm 0.0013}$ & $\mathbf{0.0579} \scriptstyle{\pm 0.0014}$ & $\mathbf{0.0242} \scriptstyle{\pm 0.0030}$ \\
& w/o tangential projection & $0.6727 \scriptstyle{\pm 0.0192}$ & $0.4482 \scriptstyle{\pm 0.0061}$ & $0.0426 \scriptstyle{\pm 0.0044}$ \\
\midrule
\multirow{2}{*}{Student-$t$ ($d=16$)}
& w/ tangential projection  & $\mathbf{0.2264} \scriptstyle{\pm 0.0476}$ & $\mathbf{0.0119} \scriptstyle{\pm 0.0029}$ & $\mathbf{0.3316} \scriptstyle{\pm 0.0298}$ \\
& w/o tangential projection & $0.5317 \scriptstyle{\pm 0.0512}$ & $0.0275 \scriptstyle{\pm 0.0027}$ & $0.3368 \scriptstyle{\pm 0.0365}$ \\
\midrule
\multirow{2}{*}{Student-$t$ ($d=32$)}
& w/ tangential projection  & $\mathbf{0.6162} \scriptstyle{\pm 0.0076}$ & $\mathbf{0.0120} \scriptstyle{\pm 0.0008}$ & $\mathbf{0.4749} \scriptstyle{\pm 0.0601}$ \\
& w/o tangential projection & $1.3792 \scriptstyle{\pm 0.3232}$ & $0.0376 \scriptstyle{\pm 0.0109}$ & $0.5338 \scriptstyle{\pm 0.0477}$ \\
\midrule
\multirow{2}{*}{Toy radial-angular}
& w/ tangential projection  & $\mathbf{0.0233} \scriptstyle{\pm 0.0049}$ & $\mathbf{0.0291} \scriptstyle{\pm 0.0292}$ & $\mathbf{0.3371} \scriptstyle{\pm 0.0421}$ \\
& w/o tangential projection & $26.7003 \scriptstyle{\pm 25.6583}$ & $0.6144 \scriptstyle{\pm 0.2793}$ & $17.0061 \scriptstyle{\pm 16.0274}$ \\
\bottomrule
\end{tabular}
\caption{Full tangential-projection ablation for RAFM. Tangential projection is not uniformly beneficial on the easiest control regimes, and can even be slightly worse on sliced Wasserstein in the mild anisotropic Gaussian control and on PIV ($d=16$). However, it becomes increasingly important as radial mismatch and dimensionality grow, substantially improving radial fidelity on Student-$t$ ($d=16,32$), PIV ($d=64$), and especially PIV ($d=256$). On the 2D toy, removing projection leads to severe degradation and numerical instability.}
\label{tab:full_projection_ablation}
\end{table*}

On all datasets except the 2D toy, runs remain numerically stable without projection, with zero NaN, exploding-norm, and invalid rates in our experiments. However, on the toy radial--angular failure mode, removing tangential projection induces non-zero NaN and invalid rates, consistent with the near-origin fragility discussed in the main text.

\section{Reproducibility and exact experimental protocol}
\label{app:reproducibility}

This appendix provides the exact implementation and evaluation protocol used to produce the reported results. Our goal is to make the experiments directly reproducible by specifying the source of truth for configurations, the software and hardware environment, the dataset construction pipeline, the checkpoint-selection rule, the aggregation protocol across seeds, and the commands used to generate the reported tables.

\subsection{Source of truth}
\label{app:source_of_truth}

All paper results were produced from the public RAFM codebase at commit \texttt{2e659c7}, including the MSGM baseline implementation stored under \texttt{baselines/}.

If a discrepancy exists between the text of the paper and a fallback default in the code, the experiment configuration file used for the run is the source of truth.
The exact configuration files used to produce the paper tables are archived under
\texttt{configs/paper/}.

\subsection{Software environment}
\label{app:software_env}

All experiments were run with the following software stack:
\begin{itemize}
    \item Python \texttt{3.10.19}
    \item PyTorch \texttt{2.6.0+cu124}
    \item CUDA \texttt{12.4} and cuDNN \texttt{9.1.0}
    \item NumPy \texttt{2.2.6}
    \item SciPy \texttt{1.15.3}
    \item scikit-learn \texttt{1.7.2}
    \item tqdm \texttt{4.65.2}
\end{itemize}

A complete frozen environment is provided in
\texttt{requirements.txt} and
\texttt{environment.yml}.
Unless otherwise stated, experiments were executed in \texttt{float32}.
The code optionally enables \texttt{torch.compile} on Linux/CUDA; this optimization is disabled on Windows.

\subsection{Hardware}
\label{app:hardware}

All reported experiments were run on
\texttt{NVIDIA RTX 2000 Ada Generation}
with
\texttt{16}
GB of GPU memory,
\texttt{AMD EPYC 9354 32-Core Processor},
and
\texttt{16}
GB of system RAM,
under
\texttt{Windows 11 (10.0.26200)}.
For timing experiments, we used the same machine for all compared methods.

\subsection{Neural architecture}
\label{app:architecture}

All Flow Matching variants (Gaussian FM, Source-only, RAFM) and the MSGM baseline use the same neural architecture in order to make the comparison as controlled as possible. The only intended differences between methods are therefore the source distribution, the path geometry, and, for MSGM, the training objective and stochastic sampler.

\paragraph{Architecture.}
The network is a multilayer perceptron (MLP) operating on the concatenation of the data vector and the scalar time variable.
Given $x \in \mathbb{R}^d$ and $t \in [0,1]$, the model input is
\[
[x; t] \in \mathbb{R}^{d+1}.
\]
The network outputs a vector field in $\mathbb{R}^d$.

More precisely, the architecture is:

\begin{center}
\begin{tabular}{ll}
\toprule
Layer 1 & Linear$(d+1, 128)$ + Swish \\
Layer 2 & Linear$(128, 128)$ + Swish \\
Layer 3 & Linear$(128, 128)$ + Swish \\
Output  & Linear$(128, d)$ \\
\bottomrule
\end{tabular}
\end{center}

\paragraph{Activation.}
We use the Swish activation
\[
\mathrm{Swish}(x) = x \, \sigma(x),
\]
where $\sigma$ denotes the logistic sigmoid.

\paragraph{Additional implementation details.}
Unless otherwise stated:
\begin{itemize}
    \item all linear layers use biases;
    \item no BatchNorm, LayerNorm, or other normalization layer is used;
    \item no dropout is used;
    \item no residual connections are used;
    \item the time variable is concatenated directly to the input, with no learned embedding and no sinusoidal embedding;
    \item weights are initialized with the default PyTorch initialization.
\end{itemize}

\paragraph{Input and output.}
The input dimension is $d+1$, where the extra coordinate corresponds to time.
The output dimension is $d$, matching the ambient data dimension.
For Flow Matching methods, this output is interpreted as the predicted velocity field $v_\theta(x,t)$.
For MSGM, the same backbone is used, but within the multiplicative-diffusion training objective.

\paragraph{Parameter counts.}
The number of trainable parameters depends on the ambient dimension $d$.
For the dimensions used in the paper, the parameter counts are:

\begin{center}
\begin{tabular}{lc}
\toprule
Dimension $d$ & Number of parameters \\
\midrule
2   & $\approx 33{,}920$ \\
16  & $\approx 35{,}712$ \\
32  & $\approx 37{,}760$ \\
256 & $\approx 66{,}304$ \\
\bottomrule
\end{tabular}
\end{center}

\paragraph{Remark.}
Using the same architecture across all compared methods is important for interpretation: improvements in the reported results should not be attributed to network capacity differences, but to the effect of radial source correction, spherical path design, and inference-time tangential projection.

\subsection{Dataset generation and preprocessing}
\label{app:dataset_generation}

\paragraph{Synthetic datasets.}
The correlated Student-$t$ and anisotropic Gaussian datasets each contain 50{,}000 samples.
For the correlated Student-$t$ benchmark, we generate
\[
X = z A^\top,
\qquad
z_i \overset{\text{i.i.d.}}{\sim} \mathrm{Student}\text{-}t(\nu),
\]
with $\nu = 3$ and dimensions $d \in \{16,32\}$.
For the anisotropic Gaussian control, we use
\[
X = z A^\top,
\qquad
z \sim \mathcal{N}(0, I_d),
\]
with the same fixed mixing matrix $A$.
The matrix $A \in \mathbb{R}^{d \times d}$ is sampled once with
\[
A_{ij} \sim \mathcal{N}(0,1),
\]
using \texttt{matrix\_seed = 42}, and is then kept fixed for all runs and all methods.
No centering, normalization, whitening, or augmentation is applied to the synthetic datasets.

\paragraph{Toy 2D dataset.}
The 2D toy dataset is generated as
\[
X = r [\cos \theta, \sin \theta]^\top,
\]
with
$r = |\mathrm{Student}\text{-}t(\nu=3)| \times \mathrm{scale}$,
$\mathrm{scale}=1$,
and $\theta$ drawn from a uniform mixture of $4$ angular modes.
Each angular mode is sampled from a Gaussian approximation to a von Mises distribution with concentration $\kappa = 5$ and mode centers uniformly spaced on $[0,2\pi)$.
The toy dataset is used only as a low-dimensional stress test and failure-mode analysis.

\paragraph{PIV dataset.}
For real data, we use the public dataset
\emph{Non-time-resolved PIV dataset of flow over a circular cylinder at Reynolds number 3900}
(DOI: \texttt{10.57745/DHJXM6}).
The exact archive used in the reported experiments is
\texttt{dataverse\_files.zip}; the canonical archive is available from the dataset DOI.

We retain all files in the archive whose filename starts with \texttt{Serie\_} and ends with \texttt{.txt}.
A frame is discarded only if:
(i) parsing fails,
(ii) the number of parsed points is not equal to $545 \times 740 = 403{,}300$,
or
(iii) \texttt{NaN} values are present in either $V_x$ or $V_y$.
In the archive used for the paper, the preprocessing script retained exactly
\texttt{998}
snapshots and skipped
\texttt{2}
snapshots.

Each retained file is parsed as a DaVis text file with rows of the form
\texttt{x;y;Vx;Vy}.
The coordinate columns $(x,y)$ are ignored after parsing; only the velocity components are used.
The data are reshaped into arrays of size $(N_y, N_x) = (740,545)$, with $x$ varying along axis $1$ and $y$ along axis $0$.

Vorticity is computed \emph{before} spatial subsampling on the full $(740 \times 545)$ grid as
\[
\omega = \frac{\partial V_y}{\partial x} - \frac{\partial V_x}{\partial y},
\]
using \texttt{numpy.gradient} with unit spatial spacing:
\[
\partial_x V_y = \texttt{numpy.gradient(Vy, axis=1)},
\qquad
\partial_y V_x = \texttt{numpy.gradient(Vx, axis=0)}.
\]
Because no physical spacing is passed to \texttt{numpy.gradient}, the resulting vorticity is expressed in velocity-per-pixel units rather than physical $\mathrm{s}^{-1}$ units.

The vorticity field is then subsampled on a regular grid using
\[
\begin{aligned}
y_\text{idx} &= \texttt{numpy.linspace(0, 739, ny, dtype=int)}, \\
x_\text{idx} &= \texttt{numpy.linspace(0, 544, nx, dtype=int)}.
\end{aligned}
\]
followed by
\[
\omega[\texttt{numpy.ix\_}(y\_\text{idx}, x\_\text{idx})].
\]
The resulting grid is flattened in row-major order.
The PIV variants used in the experiments are:
\begin{itemize}
    \item PIV $d=32$: grid $8 \times 4$
    \item PIV $d=64$: grid $8 \times 8$
    \item PIV $d=256$: grid $16 \times 16$
    \item PIV $d=16$: truncation of the first 16 coordinates of a native PIV representation, as specified in the released preprocessing code
\end{itemize}

The exact preprocessing order is:
\begin{enumerate}
    \item read \texttt{Serie\_*.txt} files from the archive,
    \item parse $(V_x, V_y)$,
    \item reject invalid frames,
    \item compute vorticity on the full grid,
    \item subsample to the target grid,
    \item flatten,
    \item stack all snapshots into an array of shape $(N,d)$,
    \item divide by $2.5$,
    \item center each dimension by subtracting the empirical mean computed over the full dataset,
    \item save the tensor as \texttt{piv\_d\{dim\}.pt} in \texttt{float32},
    \item optionally truncate dimensions at load time,
    \item re-center after truncation,
    \item finally apply the train/validation/test split.
\end{enumerate}

The PIV statistics used for centering are computed on the full dataset before the split, following the same design choice as the compared MSGM pipeline.
No oracle radial source is available for PIV.

\subsection{Train/validation/test split}
\label{app:data_split}

Unless otherwise stated, all datasets use a fixed
$60\%/20\%/20\%$
split into train, validation, and test sets.
The split is obtained from a deterministic random permutation with
\texttt{split\_seed = 0}.
For synthetic datasets with 50{,}000 samples, this yields 30{,}000 training samples, 10{,}000 validation samples, and 10{,}000 test samples.
All evaluation metrics reported in the paper are computed against the test split only.

\subsection{Training protocol}
\label{app:training_protocol}

All Flow Matching baselines (Gaussian FM, Source-only, RAFM) and the MSGM baseline use the same MLP architecture in order to isolate the effect of the source distribution and transport geometry.
The network is a 3-hidden-layer MLP with width 128 and Swish activations, taking as input the concatenation of the data vector $x \in \mathbb{R}^d$ and the scalar time $t \in [0,1]$, and outputting a vector field in $\mathbb{R}^d$.

Unless otherwise stated, all methods are trained with:
\begin{itemize}
    \item optimizer: Adam,
    \item learning rate: $10^{-3}$,
    \item $\beta_1 = 0.9$,
    \item $\beta_2 = 0.999$,
    \item $\varepsilon = 10^{-8}$,
    \item weight decay: $0$,
    \item batch size: $256$,
    \item number of optimization steps: $10{,}000$,
    \item constant learning rate schedule,
    \item no EMA,
    \item no gradient clipping,
    \item no data augmentation.
\end{itemize}

Training data are preloaded on GPU memory.
Mini-batches are sampled by direct tensor indexing with replacement using \texttt{torch.randint}, rather than through a PyTorch \texttt{DataLoader}.
For RAFM and Source-only with empirical radial source, the empirical radial distribution is estimated from the training split only.

For RAFM, training uses the matched-radius coupling described in the main paper: for each target sample $x_1$, the source radius is set to $\|x_1\|$ and only the direction is randomized during training.
The empirical radial law is required only for unconditional initialization at sampling time.

\subsection{Checkpoint selection and aggregation across seeds}
\label{app:checkpoint_selection}

All methods are trained for a fixed budget of 10{,}000 optimization steps.
Flow Matching models save checkpoints every 5{,}000 steps; MSGM checkpoints are saved every 1{,}000 steps.

The numbers reported in the paper use
\textbf{the final checkpoint at step 10{,}000}
for every method. No validation-based checkpoint selection is performed; the training budget is fixed and the last checkpoint is always used.

All reported means and standard deviations are aggregated over the same three model seeds:
\[
\{8925,\ 77395,\ 65457\},
\]
generated deterministically by
\[
\texttt{numpy.random.default\_rng(42).integers(0, 100000, size=3)}.
\]
The split seed is always fixed to $0$ and the synthetic-data matrix seed is always fixed to $42$.

\subsection{Sampling protocol}
\label{app:sampling_protocol}

For Flow Matching methods, sampling is performed by integrating the learned ODE
\[
\frac{dx}{dt} = v_\theta(x,t), \qquad t \in [0,1],
\]
starting from a source sample $x_0$.
Unless otherwise stated, we use a fixed-step RK4 solver with 128 steps, corresponding to 512 neural function evaluations.

For Gaussian FM, the initial state is sampled from $\mathcal{N}(0, I_d)$.
For empirical Source-only and RAFM, the radius is sampled from the empirical radial measure estimated on the training norms, using the empirical CDF for inversion sampling.
For oracle Source-only and oracle RAFM on synthetic datasets, the radius is sampled from the exact radial law.

For RAFM, tangential projection is applied at inference time unless stated otherwise:
\[
v_{\mathrm{proj}}(x,t)
=
v_\theta(x,t)
-
\frac{\langle x, v_\theta(x,t)\rangle}{\|x\|^2}x.
\]
In practice, the projection is skipped when $\|x\| < 10^{-3}$ in order to avoid numerical instability near the origin.

For MSGM, sampling is performed with the Stratonovich RK4 solver described in the baseline code, using 128 solver steps by default.

Unless otherwise stated, each evaluation run generates exactly 10{,}000 samples.

\subsection{Evaluation metrics}
\label{app:metrics_details}

We report radial Wasserstein-1, KS statistic, Sliced Wasserstein-1, and, when applicable, angular diagnostics and stability metrics.

\paragraph{Radial Wasserstein-1.}
Let $R_{\mathrm{gen}} = \|x_{\mathrm{gen}}\|$ and $R_{\mathrm{test}} = \|x_{\mathrm{test}}\|$.
We compute the one-dimensional Wasserstein-1 distance between the empirical norm distributions of generated and test samples.

\paragraph{KS statistic.}
We report the Kolmogorov--Smirnov statistic between the empirical CDFs of generated norms and test norms.

\paragraph{Sliced Wasserstein-1.}
We use 500 random projection directions sampled uniformly on the unit sphere.
For each direction, the projected one-dimensional Wasserstein distance is computed by sorting and matching the projected samples.
The same set of projection directions is reused across compared methods within a given evaluation run.

No fixed projection seed is used; projection directions are sampled from the current PyTorch RNG state at evaluation time.

\paragraph{Angular metrics.}
For angular evaluation, samples are partitioned into 4 radial bins defined by the test-set radial quantiles.
Within each bin, vectors are normalized to unit norm and angular sliced Wasserstein is computed with 200 random projections.
Norms are clamped to a minimum of $10^{-12}$ before normalization to avoid division by zero.

\paragraph{MMD.}
When MMD is reported, we use an RBF kernel.
The bandwidth is selected by the median heuristic computed on
the concatenation of generated and test samples (median of all nonzero pairwise distances),
and then kept fixed for the compared methods in that evaluation.

\paragraph{Stability metrics.}
We report:
(i) NaN rate,
(ii) exploding-norm rate, defined as the fraction of generated samples with norm larger than $100 \times \mathrm{median}(\|x_{\mathrm{test}}\|)$,
and
(iii) invalid rate, defined as
\[
\mathrm{invalid\_rate}
=
\mathrm{nan\_rate}
+
(1-\mathrm{nan\_rate})\times \mathrm{exploding\_rate}.
\]

\subsection{Timing protocol}
\label{app:timing_protocol}

Training-time and sampling-time measurements are obtained on the same machine and with the same software stack for all methods.

\paragraph{Training step timing.}
The reported training time per step is measured after a warm-up phase of
\texttt{10}
steps.
For CUDA runs, we call \texttt{torch.cuda.synchronize()} immediately before and after the timed region.
Each number is averaged over
\texttt{3}
independent timing repeats.

\paragraph{Sampling timing.}
Sampling time is measured for fixed NFE values in
\[
\{32, 64, 128, 256\},
\]
using batches of size
\texttt{10{,}000}
(all samples generated in a single batch).
For CUDA runs, \texttt{torch.cuda.synchronize()} is called before starting and after ending each timed sampling loop.
Compilation/JIT warm-up is excluded from the reported timing numbers.

\paragraph{Important note.}
On some first-run CUDA configurations, FM timings can be inflated by JIT and CUDA warm-up overhead.
The paper tables report post-warm-up timings.

\subsection{Exact reproduction commands}
\label{app:commands}

The exact commands used to reproduce the datasets, training runs, evaluations, and paper tables are listed below.

\paragraph{PIV preprocessing.}
\begin{verbatim}
python -m rafm.data.prepare_piv \
    --zip dataverse_files.zip \
    --out_dir data/piv \
    --grids 8x4,8x8,16x16
\end{verbatim}

\paragraph{Training.}
The exact training commands used in the paper are:
\begin{verbatim}
python -m experiments.exp1_main_benchmark \
    --config configs/exp1/<dataset>.yaml
\end{verbatim}
This trains all methods and all three seeds (8925, 77395, 65457) for the specified dataset.

\paragraph{Aggregation and tables.}
\begin{verbatim}
python scripts/generate_tables.py
python scripts/generate_figures.py
\end{verbatim}

For convenience, we also provide a single end-to-end script:
\begin{verbatim}
python scripts/run_all.py
\end{verbatim}
which reproduces all experiments, tables, and figures from raw data.

\subsection{Saved artifacts}
\label{app:saved_artifacts}

Each training run stores:
\begin{itemize}
    \item the resolved experiment configuration,
    \item the random seeds,
    \item the model checkpoint(s),
    \item the evaluation metrics in machine-readable format,
    \item the timing outputs,
    \item the generated samples used for quantitative evaluation.
\end{itemize}

Each saved run directory has the form
\begin{verbatim}
results/<dataset>/<method>/seed_<seed>/
\end{verbatim}
and contains
\texttt{config.yaml},
\texttt{metrics.json},
\texttt{timing.json},
\texttt{checkpoint\_*.pt},
and
\texttt{samples.pt}.

\subsection{What is fixed and what varies}
\label{app:fixed_vs_variable}

To make comparisons maximally controlled, the following quantities are fixed across methods unless explicitly stated otherwise:
\begin{itemize}
    \item train/validation/test split,
    \item neural architecture,
    \item optimizer and optimization hyperparameters,
    \item training budget,
    \item number of generated samples at evaluation,
    \item solver family and nominal number of steps for Flow Matching methods,
    \item aggregation over the same three seeds.
\end{itemize}

The only intended differences between the compared methods are the source distribution, the path geometry, and, for MSGM, the stochastic multiplicative-diffusion formulation and its associated training objective and sampler.

\end{document}